\documentclass{article} %
\usepackage[preprint]{colm2026_conference}

\usepackage{microtype}
\usepackage{hyperref}
\usepackage{url}
\usepackage{booktabs}
\usepackage{amsmath}
\usepackage{cancel}
\usepackage{amssymb}
\usepackage{mathtools}
\usepackage{amsthm}
\usepackage[normalem]{ulem}
\usepackage{bbm}
\usepackage{siunitx}
\usepackage{algorithm}
\usepackage{algpseudocode}
\usepackage{xspace}
\usepackage{mleftright}
\usepackage{tikz}
\usetikzlibrary{arrows.meta,positioning,fit,backgrounds,calc,decorations.pathreplacing}
\usepackage{pgfplots}
\pgfplotsset{compat=1.18}

\usepackage{amsfonts}
\usepackage[dvipsnames]{xcolor}
\usepackage{longtable}
\usepackage{bm}
\usepackage{enumitem}
\usepackage[font=small,labelfont=bf]{caption}
\usepackage{thm-restate}

\newcommand{\mymacro}[1]{#1}

\newcommand{\mcol}[2]{\ensuremath{\mathcolor{#1}{#2}}}
\definecolor{CharacterColor}{RGB}{98,115,19}

\colorlet{charC}{CharacterColor}
\colorlet{unitC}{PineGreen}
\colorlet{potC}{BrickRed}
\colorlet{propC}{ForestGreen}
\colorlet{wgtC}{Orange}
\colorlet{setC}{TealBlue}
\definecolor{TokenColor}{RGB}{140,10,89}

\newcommand{\eos}{{\textsc{eos}}}

\newcommand{\defn}[1]{\textbf{#1}}

\newcommand{\defeq}{\mathrel{\overset{\raisebox{-0.25ex}{\textnormal{\tiny def}}}{=}}}

\newcommand{\indicator}[1]{\mathbbm{1}\{ #1 \}}

\newcommand{\Weight}{\mcol{wgtC}{w}}

\newcommand{\E}{\mathop{\mathbb{E}}}
\newcommand{\ESS}[0]{\mathrm{ESS}}

\newcommand{\dist}{\mcol{propC}{q}}

\newcommand{\iid}[0]{\mathrel{\overset{\text{{{\scriptsize i.i.d.}}}}{\sim}}}           %

\newcommand{\Var}{\mathbb{V}}               %

\newcommand{\srcAlphabet}[0]{{\color{TokenColor}\ensuremath{\mathcal{X}}}\xspace}

\newcommand{\srcStrings}[0]{\ensuremath{\srcAlphabet^*}\xspace}

\newcommand{\psup}[1]{^{(#1)}}

\newcommand{\p}{\mymacro{p}}

\newcommand{\phat}{\ensuremath{\mcol{wgtC}{\widehat{p}_N}}}
\newcommand{\pbar}{\ensuremath{\mcol{wgtC}{\overline{p}_N}}}
\newcommand{\ptilde}{\ensuremath{\mcol{wgtC}{\widetilde{p}_N}}}
\newcommand{\draftStr}{{\color{TokenColor}\boldsymbol{d}}}
\newcommand{\draftSym}{{\color{TokenColor}d}}
\newcommand{\srcSymy}{{\color{TokenColor}y}}
\newcommand{\TVD}{\mathrm{TV}}
\newcommand{\BiasOne}{\mathrm{Bias}_{L^1}}
\newcommand{\BiasTwo}{\mathrm{Bias}_{L^2}}
\newcommand{\MSE}{\mathrm{MSE}}

\newcommand{\Rem}{R_N}
\newcommand{\VarTwo}{\mathrm{Var}_{L^2}}

\newcommand{\prefixLanguageOp}[1]{\overrightarrow{#1}}

\newcommand{\prefixp}{{\mymacro{\prefixLanguageOp{\p}}}}
\newcommand{\prefixq}{\prefixLanguageOp{\dist}}

\newcommand{\Kacc}{K_{\mathrm{acc}}}                      %
\newcommand{\accp}{\alpha}                           %

\newcommand{\RV}{{\color{TokenColor}\boldsymbol{X}}}
\newcommand{\srcSym}{{\color{TokenColor}x}}

\newcommand{\srcStr}
{\ensuremath{{\color{TokenColor}\boldsymbol{x}}}\xspace}

\NewDocumentCommand{\Prob}{d() m}{%
  \ensuremath{%
    \operatorname{Pr}\IfValueT{#1}{_{#1}}\!\left[#2\right]%
  }%
}
\usepackage[capitalize]{cleveref}
\crefname{section}{\S}{\S\S}
\Crefname{section}{\S}{\S\S}

\theoremstyle{plain}
\newtheorem{theorem}{Theorem}[section]
\newtheorem{proposition}[theorem]{Proposition}
\newtheorem{lemma}[theorem]{Lemma}
\newtheorem{corollary}[theorem]{Corollary}
\theoremstyle{definition}

\theoremstyle{remark}

\theoremstyle{definition}

\usepackage[disable,textsize=scriptsize]{todonotes}
\makeatletter
\newcommand*\iftodonotes{\if@todonotes@disabled\expandafter\@secondoftwo\else\expandafter\@firstoftwo\fi} 
\makeatother

\newcommand{\cc}[1]{}

\newcommand{\fix}[1]{}

\usepackage{wrapfig}
\usepackage{lineno}

\definecolor{darkblue}{rgb}{0, 0, 0.5}
\hypersetup{colorlinks=true, citecolor=darkblue, linkcolor=darkblue, urlcolor=darkblue}

\title{Faster LLM Inference via Sequential Monte Carlo}

\author{
Yahya Emara$^{*1,2}$ \quad Mauricio Barba da Costa$^{*3}$ \quad Chi-Chih Chang$^{1}$ \quad Cameron Freer$^{3}$\\
\bfseries Tim Vieira$^{4}$ \quad Ryan Cotterell$^{4}$ \quad Mohamed S. Abdelfattah$^{1,2}$\\
\textnormal{$^{1}$Cornell University \quad $^{2}$Makora \quad $^{3}$MIT \quad $^{4}$ETH Z\"{u}rich}
}

\begin{document}

\ifcolmsubmission
\linenumbers
\fi

\maketitle
\renewcommand{\thefootnote}{\fnsymbol{footnote}}
\footnotetext[1]{Co-first authorship, contact: \texttt{ye52@cornell.edu} \& \texttt{barba@mit.edu}}
\renewcommand{\thefootnote}{\arabic{footnote}}

\begin{abstract}    
Speculative decoding (SD) accelerates language model inference by drafting tokens from a cheap proposal model and verifying them against an expensive target model via rejection sampling.
Because rejection truncates the draft block at the first error, throughput degrades when draft and target diverge.
Rather than rejecting draft tokens outright, we propose to \emph{reweight} them.
To this end, we introduce sequential Monte Carlo speculative decoding (SMC-SD), which replaces token-level rejection with importance-weighted resampling over a population of draft particles. SMC-SD is a principled \emph{approximate} inference scheme that trades exactness for additional speed, while preserving theoretical bounds on its per-step approximation error.
Because LLM inference is memory bandwidth-bound, the arithmetic needed to draft particles and to score them in parallel comes nearly for free---SMC-SD uses idle compute to turn verification into a vectorized, fixed-size operation with no rollback.
Empirically, SMC-SD achieves $2.36\times$ speed-up over speculative decoding and a $5.2\times$ speed-up over autoregressive decoding, while remaining within 3\% of the target model's accuracy on reasoning, instruction-following, and coding benchmarks.\looseness=-1

\url{https://github.com/abdelfattah-lab/smcsd}
\end{abstract}

\section{Introduction}
\label{sec:introduction}
Autoregressive generation from neural language models is inherently sequential: each token depends on \emph{all} previous tokens, so generating a single sequence requires as many serial forward passes through the model as there are tokens in the output. This sequential bottleneck is the primary obstacle to faster inference.

\begin{wrapfigure}[19]{r}{0.5\linewidth}
    \centering
    \begin{tikzpicture}
    \definecolor{barAR}{RGB}{140,100,90}
    \definecolor{barSD}{RGB}{190,40,40}
    \definecolor{barSSD}{RGB}{50,110,190}
    \definecolor{barSMCSD}{RGB}{140,80,180}
    \begin{axis}[
        width=\linewidth,
        height=0.8\linewidth,
        ylabel={Throughput (tok/s)},
        ylabel style={font=\footnotesize},
        ymin=0, ymax=450,
        xmin=-0.7, xmax=5.7,
        ytick={0,100,200,300,400},
        xtick={0,1.6,3.2,4.8},
        xticklabels={SGLang\\AR, SGLang\\SD, SSD, {SMC-SD\\(\textbf{ours})}},
        xticklabel style={font=\footnotesize, align=center},
        yticklabel style={font=\footnotesize},
        ymajorgrids=true,
        grid style={gray!25},
        axis lines*=left,
        clip=false,
        xtick style={draw=none},
    ]
    \fill[barAR]    (axis cs:-0.35,0)  rectangle (axis cs:0.35,65);
    \fill[barSD]    (axis cs:1.25,0)   rectangle (axis cs:1.95,141.5);
    \fill[barSSD]   (axis cs:2.85,0)   rectangle (axis cs:3.55,225.2);
    \fill[barSMCSD] (axis cs:4.45,0)   rectangle (axis cs:5.15,335);

    \node[font=\scriptsize,            anchor=south] at (axis cs:0,   65)    {65.0};
    \node[font=\footnotesize\bfseries, anchor=south] at (axis cs:0,   100)   {1.0x};
    \node[font=\scriptsize,            anchor=south] at (axis cs:1.6, 141.5) {140.8};
    \node[font=\footnotesize\bfseries, anchor=south] at (axis cs:1.6, 176)   {2.2x};
    \node[font=\scriptsize,            anchor=south] at (axis cs:3.2, 225.2) {225.2};
    \node[font=\footnotesize\bfseries, anchor=south] at (axis cs:3.2, 260)   {3.5x};
    \node[font=\scriptsize,            anchor=south] at (axis cs:4.8, 342)   {342.0};
    \node[font=\footnotesize\bfseries, anchor=south] at (axis cs:4.8, 370)   {5.2x};
    \end{axis}
    \end{tikzpicture}
    \caption{Speed-up of SMC-SD on Llama 1B$\,\to\,$70B draft-target pair relative to autoregressive baseline, optimized tree-based SD (SGLang), Speculative Speculative Decoding \citep[SSD;][]{kumar2026speculativespeculativedecoding} on ShareGPT dataset. AR, SGLang SD, SMC-SD run on 4 H100 GPUs, while SSD runs on 5 H100 GPUs.
    }
    \label{fig:main_throughput}
\end{wrapfigure}
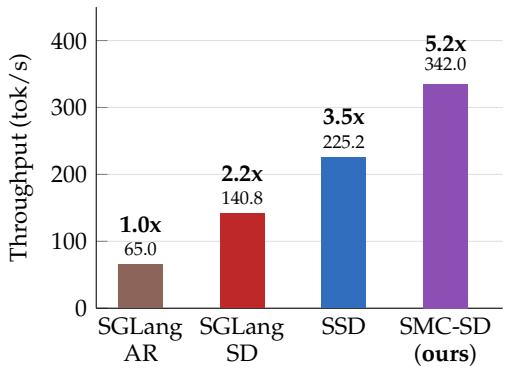

Speculative decoding \citep[SD;][]{leviathan2023fast} addresses this bottleneck by amortizing the cost of target-model calls.
At its core, SD is a rejection sampler \citep{chen2023accelerating}:
a small \emph{draft} model---chosen to be significantly cheaper to evaluate---proposes $K$-tokens, and the larger \emph{target} model verifies all $K$ tokens in a single forward pass, accepting a prefix and rejecting the rest.
This rejection criterion ensures that the resulting strings are distributed \emph{exactly} according to the target, while the speed-up factor is stochastic, depending on the alignment of the draft and target models.

This paper takes a fundamentally different approach to accelerating autoregressive generation. 
Rather than searching for a faster exact sampler, we devise an \emph{approximate} sampling scheme whose fidelity to the target can be traded off against speed.
The key modification is replacing SD's token-level rejection step with importance-weighted resampling, making the verification step an instance of sequential Monte Carlo \citep[SMC;][]{doucet2001sequential, del2006sequential, naesseth2024elementssequentialmontecarlo}. Our method \emph{reverses} the guarantee of SD---approximation quality is stochastic and the speed-up factor is deterministic. We call our method \textbf{sequential Monte Carlo speculative decoding (SMC-SD)} and describe it concretely in \Cref{sec:smc-sd}.

We derive a closed-form expression for the tokens-per-second (TPS) speed-up factor of SMC-SD with the roofline model \citep{williams2009roofline} of the GPU in \Cref{sec:speedup-analysis}. We show that the roofline model gives a simple characterization of the speed-up factor in both the memory bandwidth-bound and compute-bound regimes.

From a theory perspective, because SMC-SD is ultimately based on importance resampling, non-asymptotic bounds on its approximation error can be obtained using known techniques.
In the case of SMC-SD, we show that both the $L_2$ bias and mean squared error decay linearly in the number of particles and the $L_1$ bias decays in the root of the number of particles with constants governed by the $\chi^2$-divergence between draft and target model (\Cref{sec:bg-importance-resampling}).

From a systems perspective, SMC-SD presents a novel computation pattern that necessitates designing an inference engine around it.
In \Cref{sec:smc-sd-hardware}, we use several observations at the intersection of hardware and algorithm to guide the design of our high-throughput SMC-SD inference engine.
Both the drafting and verifying stages of SMC-SD use a vectorizable execution structure that map naturally onto GPU hardware and enable greater parallelism than standard speculative decoding. For KV cache management, we observe that all KV cache data movement during the resampling step can be replaced with efficient in-place pointer exchanges, and we reduce the size of the cache by leveraging the shared prefix among particles.

Empirically, SMC-SD achieves up to $2.5\times$ the throughput of an optimized implementation of speculative decoding on GSM8K \citep{cobbe2021trainingverifierssolvemath}, MATH500 \citep{lightman2023letsverifystepstep}, AlpacaEval \citep{dubois2025lengthcontrolledalpacaevalsimpleway}, and DS1000 \citep{lai2022ds1000naturalreliablebenchmark} across both Llama \citep{grattafiori2024llama3herdmodels} and Qwen \citep{yang2025qwen3technicalreport} model families, while remaining within 3\% of the target model's accuracy.
In a multi-GPU setting, SMC-SD achieves a $2.36\times$ speed-up over state-of-the-art speculative decoding and $5.2\times$ over the autoregressive baseline.
\Cref{fig:main_throughput} summarizes these end-to-end throughput gains.\looseness=-1

\section{Language Models and Speculative Decoding}
\label{sec:background}

\subsection{Language Modeling Background}
\label{sec:bg-ar}

Let $\srcAlphabet$ be a finite \defn{alphabet}, a finite, non-empty set.
We call elements of $\srcAlphabet$ \defn{tokens}.
We write $\srcStrings$ for the set of all finite strings over $\srcAlphabet$, including the empty string $\varepsilon$, and denote individual tokens by $\srcSym$ and strings by $\srcStr = \srcSym_1 \cdots \srcSym_T \in \srcStrings$.
We use the shorthand $\srcStr_{<t} \defeq (\srcSym_1, \ldots, \srcSym_{t-1})$ to denote the prefix up to but not including position $t$.
A \defn{language model} is a probability distribution $\p$ over $\srcStrings$.
Given a language model $\p$ over $\srcStrings$, we define the \defn{prefix probability} of a string $\srcStr$ as the total probability of all strings that begin with $\srcStr$ as follows $\prefixp(\srcStr) \defeq \sum_{\srcStr' \in \srcStrings} \p(\srcStr\srcStr')$.
 The ratio of consecutive prefix probabilities recovers a conditional distribution over the next token $\srcSym \in \srcAlphabet \cup \{\eos\}$ where $\eos \not \in \srcAlphabet$ is a distinguished symbol:

\begin{subequations}\label{eq:prefix-conditionals}
\noindent
\begin{minipage}{0.48\linewidth}
\begin{equation}
    \prefixp(\srcSymy \mid \srcStr) = \frac{\prefixp(\srcStr\, \srcSymy)}{\prefixp(\srcStr)} \label{eq:prefix-token}
\end{equation}
\end{minipage}\hfill
\begin{minipage}{0.48\linewidth}
\begin{equation}
    \prefixp(\eos \mid \srcStr) = \frac{\p(\srcStr)}{\prefixp(\srcStr)} \label{eq:prefix-eos}
\end{equation}
\end{minipage}
\end{subequations}
 Every language model admits an \defn{autoregressive factorization}:
$\p(\srcStr) = \prefixp(\eos \mid \srcStr)\prod_{t=1}^{T} \prefixp(\srcSym_t \mid \srcStr_{<t})$.
We call a language model that is autoregressively factorized an \defn{autoregressive language model}.
Finally, we say that a language model $\p$ is \defn{absolutely continuous} with respect to a second language model $\dist$, written $\p \ll \dist$, if $\dist(\srcStr) = 0 \implies \p(\srcStr) = 0, \forall \srcStr$.\looseness=-1

\subsection{Speculative Decoding}
\label{sec:bg-speculative}

Speculative decoding \citep{leviathan2023fast, chen2023accelerating} accelerates autoregressive inference by applying \defn{rejection sampling} at the token level.
We consider a \defn{target model} $\p$ and a smaller, faster-to-evaluate \defn{draft model} $\dist$ over the same alphabet $\srcAlphabet$.
Given a draft length $K$ and the draft--target pair $(\dist, \p)$, each round proceeds in two phases.
In the \textbf{draft} phase, the draft model generates a continuation $\draftStr$ of up to $K$ tokens autoregressively from $\prefixq$.
In the \textbf{verification} phase, the target model scores all $K$ draft tokens in a single forward pass.
Each token $\draftSym_j$ is accepted with probability $\accp_j = \min\bigl(1,\, \prefixp(\draftSym_j \mid \srcStr\,\draftStr_{<j}) / \prefixq(\draftSym_j \mid \srcStr\,\draftStr_{<j})\bigr)$.
On the first rejection, the token is resampled from the residual distribution and all subsequent draft tokens are discarded; 
if every token is accepted, a bonus token is drawn from $\prefixp(\cdot \mid \srcStr\,\draftStr)$.
In either case, every round produces at least one token whose distribution involves the target model.
This procedure guarantees that the output is distributed \emph{exactly} according to $\p$ \citep{leviathan2023fast}; the full pseudocode is given in \Cref{app:spec-decoding}.
The effective speed-up depends on the \emph{acceptance rate}: the closer $\dist$ is to $\p$, the more tokens survive per round.

\section{Sequential Monte Carlo Speculative Decoding}
\label{sec:smc-sd}

\begin{figure}
    \centering
    \includegraphics[width=0.9\linewidth, trim=0 3cm 0 3cm]{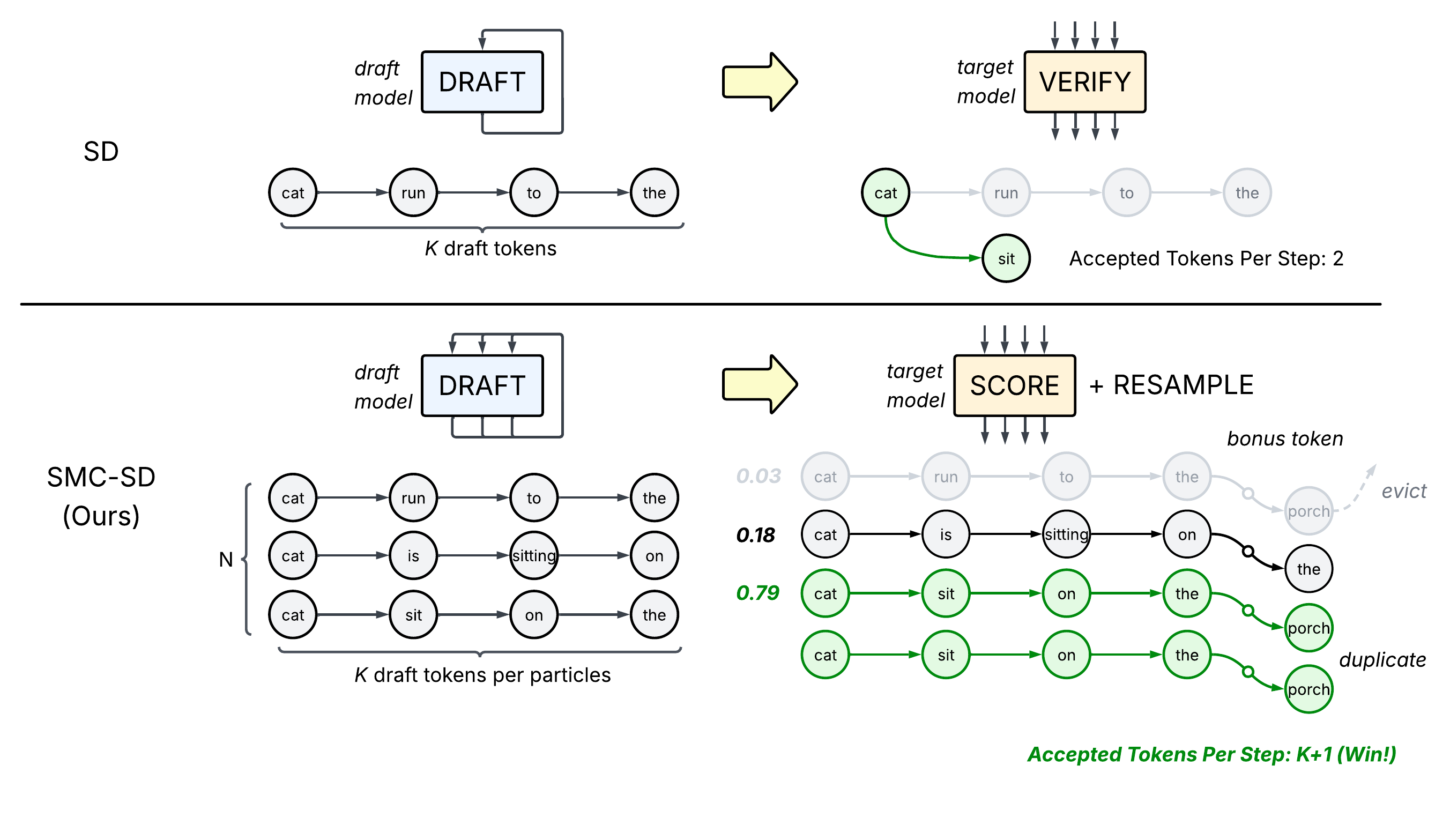}
    \caption{Our approach (bottom) compared to standard speculative decoding (top). In standard speculative decoding, a draft model generates a single sequence of draft tokens. A target model then performs a verification step, accepting a valid prefix of verified tokens while discarding rejected tokens. SMC-SD maintains a set of $N$ candidate sequences (particles). In each iteration, the draft model extends the $N$ sequences by $K$ draft tokens, where $ K$ equals 4 in the illustration above. The target model then scores these extensions and samples a bonus token. In the Resample phase, sequences with low score (e.g., $.03$) are evicted, while sequences with high probabilities (e.g., $.79$) are duplicated to form the new set of $N$ sequences.
    }
        \label{fig:method-comparison}
\end{figure}

Sequential Monte Carlo \citep[SMC;][]{doucet2001sequential, del2006sequential, naesseth2024elementssequentialmontecarlo} runs a population of $N$ particles $\{\srcStr\psup{n}\}_{n=1}^{N}$; each particle is a string that we intend to keep extending. Each particle alternates between three operations at each step $t$:
\begin{enumerate}[leftmargin=2em]
\item \textbf{Extend:} sample $\srcSym_t\psup{n} \sim \prefixq(\cdot \mid \srcStr\psup{n})$ for each particle, appending the token to $\srcStr\psup{n}$.
\item \textbf{Reweight:} update each weight by $\Weight_t\psup{n} \gets \Weight_{t-1}\psup{n} \cdot \prefixp(\srcSym_t\psup{n} \mid \srcStr\psup{n}_{<t}) / \prefixq(\srcSym_t\psup{n} \mid \srcStr\psup{n}_{<t})$.
\item \textbf{Resample:} if the effective sample size (ESS) falls below a threshold, draw new ancestor indices $a_1, \ldots, a_N$ independently from a categorical distribution over $\{1, \ldots, N\}$ with probabilities proportional to the normalized weights $\overline{\Weight}\psup{1}, \ldots, \overline{\Weight}\psup{N}$, set $\srcStr\psup{n} \gets \srcStr\psup{a_n}$ for all $n$, and reset all weights to $1/N$.
\end{enumerate}

\begin{algorithm}[t]
\caption{Sequential Monte Carlo Speculative Decoding (One Round)}\label{alg:smc-sd}
\label{alg:smc}
\begin{algorithmic}[1]
\Require Particles $\{\srcStr\psup{n}\}_{n=1}^{N}$, draft model $\dist$, target model $\p$, draft length $K$, ESS threshold $\eta \in [0,N]$
\Ensure Updated particles $\{(\srcStr\psup{n}, \Weight\psup{n})\}_{n=1}^{N}$, each extended by $K{+}1$ tokens ($K$ draft $+$ 1 bonus)
\Statex \textbf{--- Draft ---}
\For{$n = 1, \ldots, N$} \Comment{sample $N$ draft extensions in parallel}
\State $\draftStr\psup{n} \gets \varepsilon$
\For{$t = 1, \ldots, K$}
    \State $\srcSym_t \sim \prefixq(\cdot \mid \srcStr\psup{n}\, \draftStr\psup{n})$
    \State $\draftStr\psup{n} \gets \draftStr\psup{n} \srcSym_t$
\EndFor
\EndFor
\Statex \textbf{--- Score and bonus token ---}
\For{$n = 1, \ldots, N$} \Comment{single batched target forward pass}
    \State Compute $\prefixp(\draftSym^{(n)}_j \mid \srcStr\psup{n}\,\draftStr^{(n)}_{<j})$ for all $j \in \{1, \ldots, K\}$
    \State $\srcSym^+_n \sim \prefixp(\cdot \mid \srcStr\psup{n}\, \draftStr\psup{n})$ \Comment{bonus token}
\EndFor
\Statex \textbf{--- Reweight ---}
\For{$n = 1, \ldots, N$}
    \State $\Weight\psup{n} \gets \Weight\psup{n} \cdot \displaystyle\prod_{j=1}^{K} \frac{\prefixp(\draftSym^{(n)}_j \mid \srcStr\psup{n}\,\draftStr^{(n)}_{<j})}{\prefixq(\draftSym^{(n)}_j \mid \srcStr\psup{n}\, \draftStr^{(n)}_{<j})}$ \Comment{incremental block weight}
\EndFor
\State $\overline{\Weight}\psup{n} \gets \Weight\psup{n} / \sum_{m=1}^{N} \Weight\psup{m}$ for all $n$ \Comment{normalize}
\State $\ESS \gets \frac{\left(\sum_{n=1}^{N} \Weight\psup{n}\right)^2}{\sum_{n=1}^{N} (\Weight\psup{n})^2}$ \Comment{effective sample size}
\Statex \textbf{--- Resample ---}
\If{$\ESS < \eta$}
    \For{$n = 1, \ldots, N$}
        \State Sample ancestor $a_n \sim \mathrm{Cat}(\overline{\Weight}\psup{1}, \ldots, \overline{\Weight}\psup{N})$ \Comment{$\Pr[a_n = m] = \overline{\Weight}\psup{m}$}
    \EndFor
    \State $\srcStr\psup{n} \gets \srcStr\psup{a_n}\, \draftStr\psup{a_n}\, \srcSym^+_{a_n}$ for all $n \in [N]$ \Comment{duplicate high-weight particles}
    \State $\Weight\psup{n} \gets 1/N$ for all $n \in [N]$ \Comment{reset weights after resampling}
\Else
    \State $\srcStr\psup{n} \gets \srcStr\psup{n}\,\draftStr\psup{n}\,\srcSym^+_n$ for all $n \in [N]$ \Comment{keep existing ancestry}
\EndIf
\State \Return $\{(\srcStr\psup{n}, \Weight\psup{n})\}_{n=1}^{N}$ \Comment{$N$ updated particles, each extended by $K{+}1$ tokens}
\end{algorithmic}
\end{algorithm}

Resampling concentrates computation on promising prefixes and prevents weight degeneracy from rendering the particle population uninformative. The \defn{effective sample size} (ESS) \citep{huggins2014information} quantifies how many particles contribute meaningfully:
\begin{equation}\label{eq:ess}
    \ESS \;=\; \frac{\bigl(\sum_{n=1}^{N} \Weight\psup{n}\bigr)^{2}}{\sum_{n=1}^{N} \bigl(\Weight\psup{n}\bigr)^{2}}.
\end{equation}
The ESS satisfies $1 \leq \ESS \leq N$: when $\ESS \approx N$, all particles contribute meaningfully; when $\ESS$ collapses toward $1$ (\defn{weight degeneracy}), a single particle dominates. The \emph{resample} step in SMC is triggered when $\ESS$ falls below a threshold, pruning low-weight particles and duplicating high-weight ones to maintain sample diversity.

\paragraph{The SMC-SD Algorithm.}
\label{sec:smc-sd-algorithm}

We instantiate the SMC framework for speculative decoding in \Cref{alg:smc-sd}: the target is $\p$, the proposal is a cheaper draft model $\dist$, and one \defn{round} extends every particle by a draft of up to $K$ tokens.
At the start of each round, the particle population represents a set of partially generated strings $\srcStr\psup{1}, \ldots, \srcStr\psup{N}$; particles may share ancestry through earlier resampling steps, but they need not share the same current prefix.
For a weighted particle population $\{(\srcStr\psup{n}, \Weight\psup{n})\}_{n=1}^{N}$, draft length $K$, and ESS threshold $\eta \in [0, N]$, a single round proceeds as in \Cref{alg:smc-sd}. Unlike SD, no token-level accept/reject chain truncates a draft midway; every particle is extended by exactly $K$ tokens per round (\Cref{fig:method-comparison}).

Each particle independently drafts $K$ sequential tokens from its own current prefix under $\dist$, then the target model $\p$ scores all $N$ sequences in a single batched forward pass and generates a bonus token.
The resulting block weights update the particle population, and a low-ESS resampling step duplicates high-weight trajectories while pruning low-weight ones.
The round updates the internal particle population block by block; after the final round, one complete sequence is sampled from the terminal normalized weights and returned as the model output.
That is, SMC-SD maintains latent particle trajectories during generation and samples the returned sequence only at termination.

\subsection{Speed-up Analysis}
\label{sec:speedup-analysis}

We analyze SMC-SD TPS under the roofline model \citep{williams2009roofline}, which classifies a workload according to its \emph{arithmetic intensity}, or the floating-point operations (FLOPs) used per byte of memory loaded for the workload. Whether a workload is memory bandwidth-bound or compute-bound depends on whether its arithmetic intensity is greater than or less than the \emph{ridge point} of the hardware, or its peak FLOPs per memory bandwidth.
Let $B$ denote the number of independent user requests served concurrently (the \emph{batch size}), and let the \emph{effective batch size} denote $BN$, the total number of particles across all requests.
The target model's forward pass processes $BN(K{+}1)$ tokens per round: $N$ particles per request, each extended by $K$ draft tokens plus one bonus token.
Let $c = T_{\dist}/T_{\p} \ll 1$ be the draft-to-target decoding-time ratio; under the roofline model, $c \approx \rho \defeq \mathcal{P}_{\dist}/\mathcal{P}_{\p}$ (assuming the draft and target have the same quantization), since memory-bound latency is proportional to parameter count.
Let $R$ denote the effective batch size (in tokens) at which a forward pass reaches the ridge point; this depends only on the hardware and quantization format.
Below, we characterize the speed-up factor depending on whether the verification pass is memory-bound or compute-bound (proofs in \Cref{lem:speedup-mem,lem:speedup-comp}).

\emph{Memory-bound regime.}
When $BN(K{+}1) \leq R$, the target scores all $BN(K{+}1)$ tokens in a single weight-load and the speed-up over autoregressive decoding is
\begin{equation}\label{eq:speedup}
    S^{(\text{mem})} = \frac{K+1}{\rho K + 1},
\end{equation}
independent of $BN$---particles are ``free.''

\emph{Compute-bound regime.}
Once $BN(K{+}1)$ exceeds $R$, the target pass becomes compute-bound.
Writing $\hat{\rho}(BN) = \max(\rho, \rho BN/R)$ for the effective draft cost ratio, the speed-up falls to
\begin{equation}\label{eq:speedup-cb}
    S^{(\text{comp})} = \frac{K+1}{\hat{\rho}(BN)\,K + BN(K{+}1)/R},
\end{equation}
which degrades linearly in the effective batch size after the ridge point has been reached.

\Cref{fig:speedup-vs-K-N} shows the theoretical speed-up for multiple values of $N$ and $K$. The figure uses a Llama-1B$\,\to\,$8B pair. The speed-up factor increases with respect to $K$ up until the ridge point is reached; at which point, the speed-up becomes approximately constant. The speed-up factor is constant with respect to $N$  in the memory-bandwidth bound regime, but degrades linearly once the workload becomes compute-bound.

\begin{figure}[t]
\centering
\includegraphics[width=0.92\linewidth]{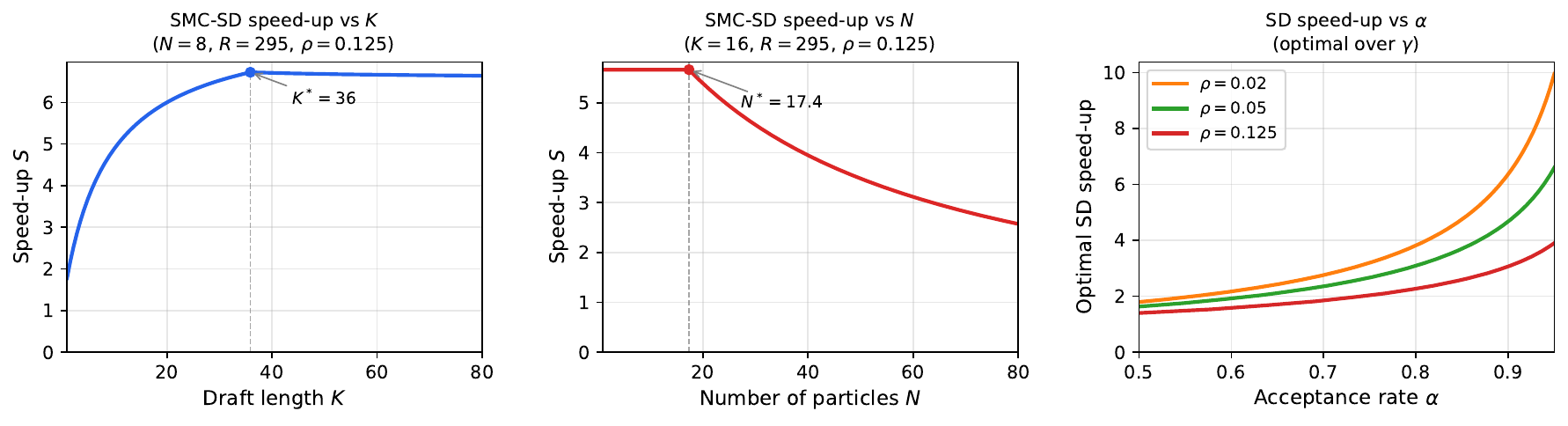}
    \caption{\textbf{Left, middle:} Theoretical speed-up of SMC-SD over autoregressive decoding for a Llama-1B$\,\to\,$8B pair as a function of draft length $K$ (left) and number of particles $N$ (middle), with $\rho = 1/8$, $B = 1$, and $R = 295$. Dashed lines mark the ridge point where the target forward pass transitions from memory-bound to compute-bound. In the memory-bound regime (left of the ridge point), increasing $K$ increases the speed-up factor and adding particles is free; past the ridge point, speed-up is approximately constant as $K$ increases, and it degrades linearly as $N$ grows. \textbf{Right:} optimal speculative decoding speed-up from \citet{leviathan2023fast}, $(1-\alpha^{\gamma+1})/((1-\alpha)(\gamma c + 1))$, as a function of acceptance rate $\alpha$ after optimizing over $\gamma$ (i.e.\ $K$), shown for several draft-to-target cost ratios $\rho$.
    }
\label{fig:speedup-vs-K-N}
\end{figure}

\subsection{Characterizing SMC-SD Approximation Error}
\label{sec:bg-importance-resampling}

Because SMC-SD is an approximate sampler, it is important to characterize its approximation quality. Fortunately, since SMC-SD is based on importance sampling, bounds on its divergence from a target distribution can be obtained using known techniques \citep{Agapiou2015ImportanceSI, del2006sequential}.
If $f$ is any evaluation metric bounded in $[0,1]$, then $|\E_{\phat}[f] - \E_{\p}[f]| \leq \|\phat - \p\|_1$: controlling the total variation distance between the SMC-SD output $\phat$ and the target $\p$ directly controls the gap in any bounded downstream metric.
We establish bounds on this distance as a function of the number of particles $N$.
We also bound the effective sample size as a function of the draft length $K$ in \ref{app:ess-proofs}, which informs the choice of operating parameters for a given compute budget.

At each round of \Cref{alg:smc-sd}, the weighted particle population defines an empirical distribution over generated sequences.
Let $\phat$ denote this \defn{importance resampling distribution}:
\begin{equation}\label{eq:phat-def}
  \phat(\srcStr) \defeq \sum_{n=1}^N \frac{\Weight\psup{n}}{\sum_{m=1}^N \Weight\psup{m}} \indicator{\srcStr\psup{n} = \srcStr}.
\end{equation}
As $N \to \infty$, since $\phat$ and $\p$ are discrete, $\phat$ converges to $\p$ in total variation distance almost surely \citep{del2006sequential}.
For finite $N$, however, $\phat$ is a \emph{random} distribution that depends on the particular draw of proposals and resampling decisions.

Let $\pbar \defeq \E[\phat]$ denote the expected importance resampling distribution, where the expectation is over all randomness in the $N$ particles.
We decompose the error of $\phat$ into \emph{bias} and \emph{variance}.
The \defn{bias} measures systematic error---how far the expected behavior $\pbar$ deviates from $\p$:

\begin{subequations}
\noindent
\begin{minipage}{0.48\linewidth}
\begin{equation}
    \BiasOne(\phat) \defeq \|\pbar - \p\|_1
\end{equation}
\end{minipage}\hfill
\begin{minipage}{0.48\linewidth}
\begin{equation}
    \BiasTwo^2(\phat) \defeq \|\pbar - \p\|_2^2
\end{equation}
\end{minipage}
\end{subequations}

\noindent The \defn{variance} captures how much a single run of $\phat$ fluctuates around its mean $\pbar$:
\begin{subequations}
\noindent
\begin{minipage}{0.48\linewidth}
\begin{equation}
    \mathrm{Var}_{L^1}(\phat) \defeq \E\bigl[\|\phat - \pbar\|_1\bigr]
\end{equation}
\end{minipage}\hfill
\begin{minipage}{0.48\linewidth}
\begin{equation}
    \VarTwo(\phat) \defeq \E\bigl[\|\phat - \pbar\|_2^2\bigr]
\end{equation}
\end{minipage}
\end{subequations}
The \defn{mean squared error} combines both sources of error and decomposes via the bias--variance identity:
\begin{equation}
  \MSE(\phat) \defeq \E \bigl[\|\phat - \p\|_2^2\bigr] = \underbrace{\E\bigl[\|\phat - \pbar\|_2^2\bigr]}_{\VarTwo} + \underbrace{\|\pbar - \p\|_2^2}_{\BiasTwo^2}.
\end{equation}

\begin{restatable}{theorem}{sirtvbound}
\label{prop:sir-tv}
Let $\srcStr\psup{1}, \ldots, \srcStr\psup{N} \iid \dist$ with $\p \ll \dist$ and $\E_{\dist}[W^4] < \infty$ where $W \defeq \p(\RV)/\dist(\RV)$. Let $\phat$ be the importance resampling distribution, $\pbar \defeq \E[\phat]$, and $\chi^2(\p \| \dist) \defeq \E_{\dist}[(W - 1)^2]$.
Then the following inequalities hold:
\begin{enumerate}[leftmargin=2em]
\item $\BiasOne(\phat) = O\left(\sqrt{(1 + \chi^2(\p \| \dist)) / N}\right)$.
\item $\BiasTwo^2(\phat) = \|\pbar - \p\|_2^2 = O\left((1 + \chi^2(\p \| \dist)) / N\right)$.
\item $\MSE(\phat) = \E\bigl[\|\phat - \p\|_2^2\bigr] = O\left((1 + \chi^2(\p \| \dist)) / N\right)$.
\end{enumerate}
\end{restatable}

\Cref{prop:sir-tv} bounds the error of importance resampling (proof and supporting lemmata in \Cref{app:proofs}), the primitive that each round of SMC-SD applies to its particle population. For a single round, both the $L^2$ bias and mean squared error decay as $O(1/N)$ in the number of particles, and the $L^1$ bias decays as $O(1/\sqrt{N})$, with constants governed by the $\chi^2$-divergence between draft and target. These per-round bounds give practitioners a concrete lever---increasing $N$ tightens the approximation at each step---but they do not directly characterize the error of the full multi-round SMC procedure, where resampling introduces inter-particle correlations and path degeneracy may accumulate across rounds. We leave bounding the error of the resampling step of SMC-SD to future work. We validate empirically in \Cref{sec:experiments} that the per-round guarantees translate into strong end-to-end accuracy across benchmarks.

\subsection{High-Throughput SMC-SD Engine Design}
\label{sec:smc-sd-hardware}

SMC-SD is not merely an algorithm that happens to run on GPUs---its vectorized structure maps onto the batch-parallel execution model of modern accelerators and makes use of modern optimizations on inference engines including dynamic batching, PagedAttention, and RadixAttention \citep{280922,kwon2023vllm,zheng2024sglangefficientexecutionstructured}. As such, we build our SMC-SD engine as a fork of SGLang \citep{zheng2024sglangefficientexecutionstructured}. We describe our high-throughput inference engine for SMC-SD, organized around a series of observations at the intersection of algorithm and hardware.  

\paragraph{Observation 1: SMC-SD converts under-utilized compute into speed.}
Autoregressive decoding is heavily memory bandwidth-bound: each token requires
loading the full model weights yet performs only a single vector--matrix
multiply. The \emph{arithmetic intensity}, the FLOPs performed per byte of memory load, for a forward pass of a language model quantized to $b$ bytes per parameter remains at just $\mathrm{AI}_{\mathrm{AR}} = 2/b$. Standard speculative decoding increases the arithmetic intensity by a factor of $K$
but both methods remain far below the roofline ridge points of modern
GPUs---312 for A100, 412 for H200, 584 for B300 in
FP16~\citep{nvidia_a100_datasheet,nvidia_h100_datasheet,nvidia_h200_datasheet,nvidia_blackwell_datasheet}---leaving the vast majority
of available compute idle.

SMC-SD closes this gap by leveraging parallelism along the \emph{batch-wise} and \emph{sequence-wise} axes. The draft model uses dynamic batching to concurrently generate $N$ tokens at a time in lockstep, followed by the target model scoring a block of $N(K{+}1)$ tokens using a single re-prefill step.
Let $\mathcal{P}$ denote the target model's parameter count, $b$ the bytes per parameter ($b{=}2$ for FP16, $b{=}1$ for INT8), and $\rho = \mathcal{P}_{\dist}/\mathcal{P}_{\p}$ the draft-to-target parameter ratio.
Over a full SMC-SD cycle, the arithmetic intensity is (proofs in \Cref{app:smc-sd-analysis})
\begin{equation}\label{eq:ai-cycle}
    \text{AI}_{\text{cycle}} = \frac{2N(\rho K{+}K{+}1)}{b(\rho K + 1)} \;\approx\; \frac{2N(K{+}1)}{b} \quad\text{when } \rho \ll 1,
\end{equation}

an $\approx N\times$ increase over standard SD.
For Llama-3.1-70B in FP16 with $N{=}8$ and $K{=}16$, $\text{AI}_{\text{cycle}} \approx 136$ compared to just $17$ for standard SD with the same $K$.

\paragraph{Observation 2: Pointer operations turn prefix-sharing into compute and memory gains.}
The resampling step in SMC-SD introduces a novel challenge in inference engine design. When resampling duplicates a high-weight particle, a na\"ive implementation would require KV-cache data movement across particles and \textit{re-prefill} for the duplicated sequences. In SMC-SD, we avoid both of these costs by exploiting the prefix-sharing structure together with page-based KV-cache management \citep{kwon2023vllm} and RadixAttention \citep{zheng2024sglangefficientexecutionstructured}. Specifically, duplicated particles share exactly the same prefix as their ancestors up to the resampling point, so the engine can perform resampling by copying page metadata and incrementing the reference counts rather than copying the underlying KV tensors directly. As a result, resampling is implemented as an $O(\ell_{\text{seq}})$ page-metadata operation, rather than an $O(\ell_{\text{seq}} \times d_{\text{model}})$ KV-tensor copy. 
Exploiting prefix sharing not only avoids KV-cache data movement and reduces FLOPs, but also lowers memory usage. Empirically, for the Llama-1B$\,\to\,$8B, $N{=}8$, $K{=}16$ configuration, our optimization yields a \textbf{72.3\%} reduction in KV-cache usage.

\section{Experiments}
\label{sec:experiments}

We characterize the \emph{speed--accuracy Pareto frontier} of SMC-SD on a single H100 GPU across GSM8K, MATH500, AlpacaEval, and DS1000 \citep{dubois2025lengthcontrolledalpacaevalsimpleway, cobbe2021trainingverifierssolvemath, lai2022ds1000naturalreliablebenchmark, lightman2023letsverifystepstep}.
We test two model families---Llama~3.2-1B$\to$3.1-8B, and Qwen~0.5B$\to$14B (SD) versus Qwen~3B$\to$14B (SMC-SD)\footnote{For Qwen, the 0.5B draft makes for a faster speculative decoder than 3B because the smaller draft is faster and SD's acceptance rates do not justify the larger model. SMC-SD, by contrast, benefits from the stronger 3B proposal, so we compare each method with its best draft model.}---to demonstrate that SMC-SD's gains are consistent across architectures and draft-to-target size ratios.

\paragraph{GSM8K, MATH500, AlpacaEval, and DS1000.}
\Cref{fig:pareto} shows the speed--accuracy Pareto frontier for SMC-SD and optimized SD on GSM8K, MATH500, AlpacaEval, and DS1000.
Across all four benchmarks and both model families, the SMC-SD frontier lies above and to the right of the SD baseline, indicating simultaneous improvements in throughput and accuracy or favorable tradeoffs between the two.
\Cref{tab:speedup-summary} quantifies the gains at two operating points: \emph{iso-accuracy} (the fastest SMC-SD configuration within 3 percentage points of SD's accuracy) and \emph{max speed} (within 10 percentage points).
At iso-accuracy, SMC-SD achieves $1.1$--$2.5\times$ the throughput of SD; allowing a modest accuracy relaxation widens this to up to $3.4\times$ (\Cref{app:speedup-table}).

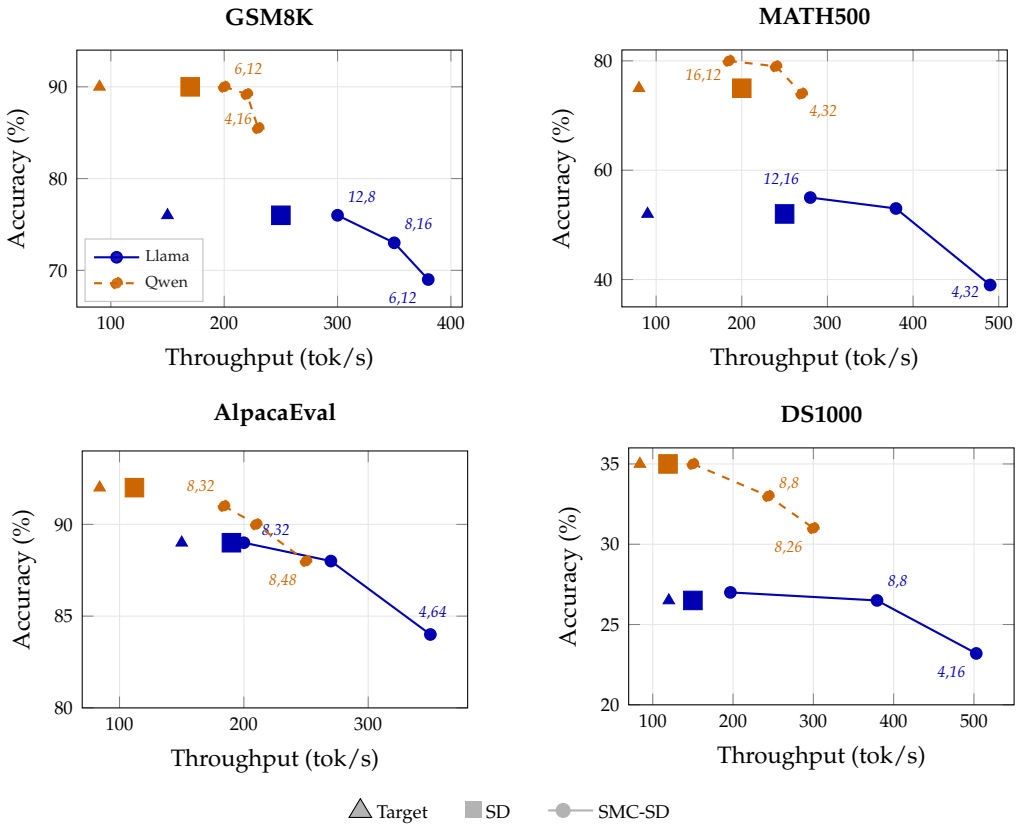
\begin{figure}[t]
\centering
\begin{minipage}{0.48\textwidth}
\centering
\begin{tikzpicture}
\begin{axis}[
    title={\textbf{GSM8K}},
    width=\textwidth, height=5cm,
    xlabel={Throughput (tok/s)},
    ylabel={Accuracy (\%)},
    xmin=70, xmax=410,
    ymin=66, ymax=94,
    grid=major, grid style={gray!20},
    legend style={font=\tiny, at={(0.02,0.02)}, anchor=south west,
                  draw=gray!50, fill=white, fill opacity=0.9, cells={anchor=west}},
    every axis label/.style={font=\small},
    tick label style={font=\scriptsize},
    title style={font=\small},
]
\addplot[forget plot, only marks, mark=triangle*, mark size=2.5pt, blue!70!black]
    coordinates {(150, 76)};

\addplot[forget plot, only marks, mark=square*, mark size=3.5pt, blue!70!black]
    coordinates {(250, 76)};

\addplot[thick, mark=*, mark size=2pt, blue!70!black]
    coordinates {(300, 76) (350, 73) (380, 69)};
\node[font=\tiny, blue!70!black, anchor=south west] at (axis cs:301, 76.2) {\textit{12,8}};
\node[font=\tiny, blue!70!black, anchor=south west] at (axis cs:351, 73.2) {\textit{8,16}};
\node[font=\tiny, blue!70!black, anchor=north east] at (axis cs:378, 68.7) {\textit{6,12}};

\addplot[forget plot, only marks, mark=triangle*, mark size=2.5pt, orange!80!black]
    coordinates {(90, 90)};

\addplot[forget plot, only marks, mark=square*, mark size=3.5pt, orange!80!black]
    coordinates {(170, 90)};

\addplot[thick, mark=*, mark size=2pt, orange!80!black, dashed]
    coordinates {(200, 90) (220, 89.2) (230, 85.5)};
\node[font=\tiny, orange!80!black, anchor=south west] at (axis cs:201, 90.1) {\textit{6,12}};
\node[font=\tiny, orange!80!black, anchor=south west] at (axis cs:221, 88.25) {};
\node[font=\tiny, orange!80!black, anchor=north east] at (axis cs:233, 88.2) {\textit{4,16}};

\addlegendentry{Llama}
\addlegendentry{Qwen}
\end{axis}
\end{tikzpicture}
\end{minipage}%
\hfill
\begin{minipage}{0.48\textwidth}
\centering
\begin{tikzpicture}
\begin{axis}[
    title={\textbf{MATH500}},
    width=\textwidth, height=5cm,
    xlabel={Throughput (tok/s)},
    ylabel={Accuracy (\%)},
    xmin=60, xmax=510,
    ymin=35, ymax=82,
    grid=major, grid style={gray!20},
    every axis label/.style={font=\small},
    tick label style={font=\scriptsize},
    title style={font=\small},
]
\addplot[forget plot, only marks, mark=triangle*, mark size=2.5pt, blue!70!black]
    coordinates {(90, 52)};

\addplot[forget plot, only marks, mark=square*, mark size=3.5pt, blue!70!black]
    coordinates {(250, 52)};

\addplot[forget plot, thick, mark=*, mark size=2pt, blue!70!black]
    coordinates {(280, 55) (380, 53) (490, 39)};
\node[font=\tiny, blue!70!black, anchor=south east] at (axis cs:278, 55.3) {\textit{12,16}};
\node[font=\tiny, blue!70!black, anchor=north east] at (axis cs:488, 40.5) {\textit{4,32}};

\addplot[forget plot, only marks, mark=triangle*, mark size=2.5pt, orange!80!black]
    coordinates {(80, 75)};

\addplot[forget plot, only marks, mark=square*, mark size=3.5pt, orange!80!black]
    coordinates {(200, 75)};

\addplot[forget plot, thick, mark=*, mark size=2pt, orange!80!black, dashed]
    coordinates {(185, 80) (240, 79) (270, 74)};
\node[font=\tiny, orange!80!black, anchor=north east] at (axis cs:187, 80.2) {\textit{16,12}};
\node[font=\tiny, orange!80!black, anchor=south west] at (axis cs:242, 79.2) {};
\node[font=\tiny, orange!80!black, anchor=north west] at (axis cs:268, 73.7) {\textit{4,32}};
\end{axis}
\end{tikzpicture}
\end{minipage}

\vspace{4pt}

\begin{minipage}{0.48\textwidth}
\centering
\begin{tikzpicture}
\begin{axis}[
    title={\textbf{AlpacaEval}},
    width=\textwidth, height=5cm,
    xlabel={Throughput (tok/s)},
    ylabel={Accuracy (\%)},
    xmin=70, xmax=380,
    ymin=80, ymax=94,
    grid=major, grid style={gray!20},
    every axis label/.style={font=\small},
    tick label style={font=\scriptsize},
    title style={font=\small},
]
\addplot[forget plot, only marks, mark=triangle*, mark size=2.5pt, blue!70!black]
    coordinates {(150, 89)};
\addplot[forget plot, only marks, mark=square*, mark size=3.5pt, blue!70!black]
    coordinates {(190, 89)};

\addplot[thick, mark=*, mark size=2pt, blue!70!black]
    coordinates {(200, 89) (270, 88) (350, 84)};
\node[font=\tiny, blue!70!black, anchor=south west] at (axis cs:207, 88.8) {\textit{8,32}};
\node[font=\tiny, blue!70!black, anchor=south] at (axis cs:352, 84.3) {\textit{4,64}};

\addplot[forget plot, only marks, mark=triangle*, mark size=2.5pt, orange!80!black]
    coordinates {(84, 92)};
\addplot[forget plot, only marks, mark=square*, mark size=3.5pt, orange!80!black]
    coordinates {(112, 92)};

\addplot[thick, mark=*, mark size=2pt, orange!80!black, dashed]
    coordinates {(184, 91) (210, 90) (250, 88)};
\node[font=\tiny, orange!80!black, anchor=south east] at (axis cs:184, 91.2) {\textit{8,32}};
\node[font=\tiny, orange!80!black, anchor=south west] at (axis cs:210, 90.2) {};
\node[font=\tiny, orange!80!black, anchor=north east] at (axis cs:250, 87.8) {\textit{8,48}};
\end{axis}
\end{tikzpicture}
\end{minipage}%
\hfill
\begin{minipage}{0.48\textwidth}
\centering
\begin{tikzpicture}
\begin{axis}[
    title={\textbf{DS1000}},
    width=\textwidth, height=5cm,
    xlabel={Throughput (tok/s)},
    ylabel={Accuracy (\%)},
    xmin=70, xmax=550,
    ymin=20, ymax=36,
    grid=major, grid style={gray!20},
    every axis label/.style={font=\small},
    tick label style={font=\scriptsize},
    title style={font=\small},
]
\addplot[forget plot, only marks, mark=triangle*, mark size=2.5pt, blue!70!black]
    coordinates {(120, 26.5)};
\addplot[forget plot, only marks, mark=square*, mark size=3.5pt, blue!70!black]
    coordinates {(150, 26.5)};

\addplot[forget plot, only marks, mark=triangle*, mark size=2.5pt, orange!80!black]
    coordinates {(84, 35)};
\addplot[forget plot, only marks, mark=square*, mark size=3.5pt, orange!80!black]
    coordinates {(119, 35)};

\addplot[thick, mark=*, mark size=2pt, blue!70!black]
    coordinates {(196.7, 27) (379.3, 26.5) (502.9, 23.2)};
\node[font=\tiny, blue!70!black, anchor=south west] at (axis cs:381, 26.7) {\textit{8,8}};
\node[font=\tiny, blue!70!black, anchor=north east] at (axis cs:501, 23.0) {\textit{4,16}};

\addplot[thick, mark=*, mark size=2pt, orange!80!black, dashed]
    coordinates {(150, 35) (244, 33) (300, 31)};
\node[font=\tiny, orange!80!black, anchor=south west] at (axis cs:246, 32.8) {\textit{8,8}};
\node[font=\tiny, orange!80!black, anchor=north east] at (axis cs:298, 30.8) {\textit{8,26}};
\end{axis}
\end{tikzpicture}
\end{minipage}

\vspace{2pt}
\centering
\begin{tikzpicture}
\node[font=\scriptsize] at (0,0) {%
  \tikz{\draw[fill=gray!60] (0,0.25) -- (0.15,0) -- (-0.15,0) -- cycle;}~Target\qquad
  \tikz{\fill[gray!60] (0,0) rectangle (0.22,0.22);}~SD\qquad
  \tikz{\fill[gray!60] (0,0.11) circle (0.1);\draw[thick,gray!60] (-0.2,0.11)--(0.4,0.11);}~SMC-SD
};
\end{tikzpicture}
\caption{Speed--accuracy Pareto frontier on GSM8K (top left), MATH500 (top right), AlpacaEval (bottom left), and DS1000 (bottom right).
Blue/solid: Llama~3.2-1B $\to$ 3.1-8B;\; orange/dashed: Qwen (SD: 0.5B$\to$14B; SMC-SD: 3B$\to$14B).
Italic labels show $(N,K)$ configurations along the SMC-SD Pareto frontier. Experiments were conducted using a single H100 SXM GPU. 
}
\label{fig:pareto}
\end{figure}

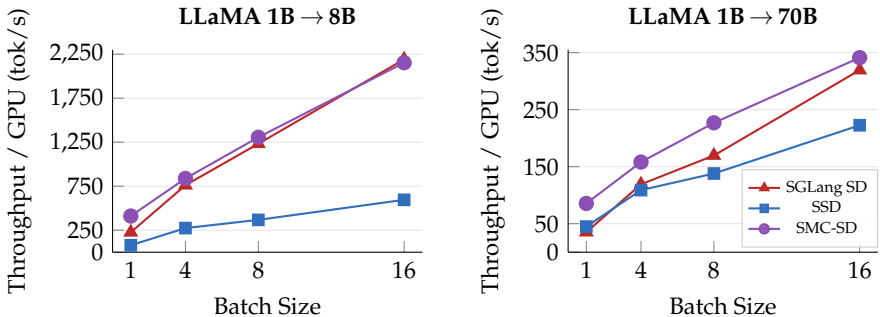
\begin{figure}[t]
\centering
\definecolor{colSD}{RGB}{190,40,40}
\definecolor{colSSD}{RGB}{50,110,190}
\definecolor{colSMCSD}{RGB}{140,80,180}
\begin{minipage}{0.85\linewidth}
\centering
\begin{minipage}{0.48\linewidth}
\centering
\begin{tikzpicture}
\begin{axis}[
    title={\textbf{LLaMA 1B\,$\to$\,8B}},
    title style={font=\small},
    width=\linewidth,
    height=0.75\linewidth,
    xlabel={Batch Size},
    ylabel={Throughput / GPU (tok/s)},
    xlabel style={font=\footnotesize},
    ylabel style={font=\footnotesize},
    xmin=0, xmax=17,
    ymin=0, ymax=2300,
    xtick={1,4,8,16},
    ytick={0,250,750,1250,1750,2250},
    xticklabel style={font=\footnotesize},
    yticklabel style={font=\footnotesize},
    ymajorgrids=true,
    grid style={gray!20},
    axis lines*=left,
    clip=false,
]
\addplot[thick, mark=triangle*, mark size=2.5pt, colSD, mark options={fill=colSD}]
    coordinates {(1,226.4) (4,761.6) (8,1232.5) (16,2191.1)};
\addplot[thick, mark=square*, mark size=2pt, colSSD, mark options={fill=colSSD}]
    coordinates {(1,79.1) (4,274.0) (8,367.3) (16,595.2)};
\addplot[thick, mark=*, mark size=2.5pt, colSMCSD, mark options={fill=colSMCSD}]
    coordinates {(1,410.91) (4,838) (8,1305.9) (16,2153.1)};
\end{axis}
\end{tikzpicture}
\end{minipage}%
\hfill
\begin{minipage}{0.48\linewidth}
\centering
\begin{tikzpicture}
\begin{axis}[
    title={\textbf{LLaMA 1B\,$\to$\,70B}},
    title style={font=\small},
    width=\linewidth,
    height=0.75\linewidth,
    xlabel={Batch Size},
    ylabel={Throughput / GPU (tok/s)},
    xlabel style={font=\footnotesize},
    ylabel style={font=\footnotesize},
    xmin=0, xmax=17,
    ymin=0, ymax=355,
    xtick={1,4,8,16},
    ytick={0,50,150,250,350},
    xticklabel style={font=\footnotesize},
    yticklabel style={font=\footnotesize},
    ymajorgrids=true,
    grid style={gray!20},
    axis lines*=left,
    clip=false,
    legend style={
        at={(0.56,0.41)},
        anchor=north west,
        font=\tiny,
        draw=gray!40,
        fill=white,
        fill opacity=0.9,
        text opacity=1,
        row sep=-2pt,
        inner sep=2pt,
        legend image post style={xscale=0.7, yscale=0.7},
    },
]
\addplot[thick, mark=triangle*, mark size=2.5pt, colSD, mark options={fill=colSD}]
    coordinates {(1,35.2) (4,119.1) (8,169.6) (16,319.4)};
\addlegendentry{SGLang SD}
\addplot[thick, mark=square*, mark size=2pt, colSSD, mark options={fill=colSSD}]
    coordinates {(1,45.0) (4,108.9) (8,138.0) (16,222.6)};
\addlegendentry{SSD}
\addplot[thick, mark=*, mark size=2.5pt, colSMCSD, mark options={fill=colSMCSD}]
    coordinates {(1,85.5) (4,158.1) (8,226.9) (16,341.2)};
\addlegendentry{SMC-SD}
\end{axis}
\end{tikzpicture}
\end{minipage}
\end{minipage}
\caption{Throughput per GPU vs.\ batch size on ShareGPT for SGLang SD, SSD, and SMC-SD.
1B$\to$8B uses 1~GPU for SGLang~SD and SMC-SD, 2~GPUs for SSD.
1B$\to$70B uses 4~H100 GPUs for SGLang~SD and SMC-SD, 5~H100 GPUs for SSD.
SMC-SD achieves the highest per-GPU throughput across most batch sizes for both model pairs. 
}
\label{fig:batch-throughput}
\end{figure}

Despite being an approximate sampler, SMC-SD maintains strong downstream accuracy: every point on the Pareto frontier in \Cref{fig:pareto} remains within 15\% of the target model's accuracy. The throughput gains come from a structural difference in how SMC-SD uses the target forward pass. In standard SD, the first rejected token truncates the draft and triggers a KV-cache rollback, so the number of tokens produced per round is variable and often much less than $K{+}1$. SMC-SD always produces exactly $K{+}1$ tokens per particle per round---there is nothing to roll back.

\paragraph{Multi-GPU throughput.}
We deploy SMC-SD on a Llama~3.2-1B$\,\to\,$Llama~3-70B draft-target pair across 4~H100 GPUs.
\Cref{fig:main_throughput} summarizes the results: SMC-SD achieves 342~tok/s, a $5.2\times$ speed-up over the autoregressive baseline and a $2.36\times$ speed-up over optimized tree-based speculative decoding (SGLang SD).
\Cref{fig:batch-throughput} shows throughput per GPU as a function of batch size for both the 1B$\to$8B and 1B$\to$70B pairs; SMC-SD leads across most batches tested.

\section{Related Work}
\label{app:related-work}

\paragraph{SMC for Language Models.}
Recent work has applied SMC to large language models for \emph{constrained generation}: steering outputs toward syntactic constraints \citep{loula2025syntactic, xefteri-etal-2025-syntactic}, semantic specifications \citep{lew2023sequential}, or reward-defined posteriors \citep{zhao2024smc}.
In these settings, the target is a posterior $\p(x \mid \text{constraint})$ and the proposal is the unconstrained language model.
SMC-SD repurposes the same machinery for \emph{inference acceleration}: the target is $\p$ itself and the proposal is a faster draft model $\dist$.

\paragraph{Exact speculative decoding.}
Tree decoding algorithms \citep{cai2024medusasimplellminference, li2025eagle3} extend standard SD by proposing a tree of tokens at each position, increasing the expected number of accepted tokens while preserving exact sampling.
SMC-SD trades exactness for speed via a tunable approximation.

\paragraph{Approximate speculative decoding.}
Several methods relax exactness through heuristic acceptance criteria.
FSD \citep{holsman2025fuzzy} accepts tokens based on draft--target divergence;
FLy \citep{li2025fly} uses an entropy gate with a deferred acceptance window;
Judge Decoding \citep{bachmann2025judgedecoding} trains a judge network to replace rejection sampling.
All retain a per-token accept/reject gate: a single rejection truncates the block, each token's contribution is binary, and outputs are variable-length.
SMC-SD instead assigns each \emph{block} a continuous importance weight, produces a fixed number of tokens per round, and provides asymptotic convergence guarantees---all without additional training.

\paragraph{Complementary strategies.}
Several orthogonal techniques can be combined with SMC-SD.
On the \emph{drafting} side, EAGLE \citep{li2024eagle2fasterinferencelanguage, li2025eagle3, li2025eaglespeculativesamplingrequires} conditions the draft model on target-model embeddings, Medusa \citep{cai2024medusasimplellminference} attaches parallel decoding heads, and DFlash \citep{chen2026dflashblockdiffusionflash} uses a diffusion model---all producing higher-quality or faster proposals that would directly benefit SMC-SD.
On the \emph{scheduling} side, Specdec++ \citep{huang2025specdec} adaptively modulates speculation length and SSD \citep{kumar2026speculativespeculativedecoding} asynchronously overlaps drafting and verifying; analogous adaptive strategies over $N$ and $K$ are a natural extension of SMC-SD. All results use the optimal temperature parameters for each task for both SMC-SD and SD.

\section{Conclusion}

We presented SMC-SD, an algorithm that accelerates LLM inference by replacing speculative decoding's token-level rejection with importance-weighted resampling over a population of draft particles.
We established non-asymptotic error bounds showing that the bias and MSE of the importance sampling component of SMC-SD decays as $O(1/N)$, and we demonstrated empirically that SMC-SD achieves $2.36\times$ the throughput of optimized speculative decoding and $5.2\times$ over autoregressive decoding in a multi-GPU setting, while remaining within 3\% of the target model's accuracy across reasoning, instruction-following, and coding benchmarks.

Beyond acceleration, a key property of SMC-SD is that the same importance-weighting machinery can target distributions beyond $\p$ itself.
Given a factor $\Psi$, SMC-SD can sample approximately from distributions of the form $\pi(x_{1:T}) \propto \p(x_{1:T}) \cdot \Psi(x_{1:T})$, provided one can define a sequence of potential functions $\{\psi_t\}_{t=1}^{T}$ from which properly weighted pairs $(\Weight_t, x_t) \sim \psi_t$ can be drawn \citep{del2006sequential}. The usual SMC convergence guarantees carry over. Three natural instantiations are:
\begin{itemize}[leftmargin=2em, itemsep=4pt]
  \item \textbf{Power sampling} ($\Psi = p^{\alpha-1}$, targeting $\pi \propto p^\alpha$). Unlike per-token temperature scaling, this tempers the distribution over \emph{entire sequences}. \Cref{app:power-sampling} shows a 10\% Pass@1 improvement on MATH500 with Llama-1B$\,\to\,$8B and a dominating Pass@$k$ for $k>1$.
  \item \textbf{Constrained generation} ($\Psi = \mathbf{1}[x \in \mathcal{A}]$). Hard constraints---syntactic validity, format compliance, length bounds---can be enforced while preserving the target's distribution conditioned on the constraint. Prior work \citep{lew2023sequential, loula2025syntactic, xefteri-etal-2025-syntactic} has shown that such constraints can improve accuracy on coding benchmarks; incorporating them into SMC-SD could simultaneously increase downstream performance while accelerating inference.
  \item \textbf{Reward-weighted decoding} ($\Psi = \exp(\beta \cdot R(x))$, targeting $\pi \propto \p \cdot e^{\beta R}$). This targets the KL-regularized optimal policy arising in RLHF without fine-tuning, which is relevant for agentic settings where the same base model must be steered toward different objectives.
\end{itemize}
Standard speculative decoding cannot express any of these: it requires knowledge of the normalizing constant of the target, which is unavailable for this class of distributions.

On the systems side, much work remains to achieve the theoretical roofline model speedup of SMC-SD demonstrated in \Cref{fig:speedup-vs-K-N}. 
First, there are opportunities to overlap resampling with drafting to remove its performance overhead altogether. 
This can be done by preemptively speculating the result of the ESS threshold, and beginning the next round of drafting accordingly.
Furthermore, the draft and target models can be disaggregated across hardware, and asynchronously dispatch drafting and scoring calls. 
Second, by allowing $N$ and $K$ to vary dynamically, SMC-SD can adapt to underlying hardware and workload characteristics, enabling runtime exploration of its tunable Pareto frontier.
Finally, SMC-SD reflects a trend in co-designing inference algorithms with asymmetric hardware scaling---compute grows faster than memory bandwidth~\citep{zadouri2026flashattention4algorithmkernelpipelining}. For instance, the Blackwell GPU architecture significantly increases FLOPS over the previous generation of Hopper GPU and keeps memory bandwidth approximately equal; this increased compute budget makes it an ideal target for SMC-SD.

\section*{Author Contributions}

\begin{itemize}
    \item \textbf{Yahya Emara} (\texttt{ye52@cornell.edu}): Lead system design/inference engine development, experiment design, formal analysis (arithmetic intensity), visualization, writing.
    \item \textbf{Mauricio Barba da Costa} (\texttt{barba@mit.edu}): Research conception, formal analysis (speed-up, arithmetic intensity, approximation error), experiment design (prototype, power sampling), visualization, writing.
    \item \textbf{Chi-Chih Chang} (\texttt{cc2869@cornell.edu}): System design/inference engine development, visualization, writing.
    \item \textbf{Cameron Freer} (\texttt{freer@mit.edu}): Senior project leadership, project advising and mentorship, technical advice.
     \item \textbf{Tim Vieira} (\texttt{tim.f.vieira@gmail.com}): Senior project leadership, project advising and mentorship, technical advice.
    \item \textbf{Ryan Cotterell} (\texttt{ryan.cotterell@inf.ethz.ch}): Formal analysis (approximation error), senior project leadership, project advising and mentorship, technical advice, writing.
    \item \textbf{Mohamed Abdelfattah} (\texttt{mohamed@cornell.edu}): Senior project leadership, formal analysis (speed-up, arithmetic intensity analysis), system design, project narrative development, project advising and mentorship, technical advice, visualization, writing.
\end{itemize}

\bibliography{colm2026_conference}
\bibliographystyle{colm2026_conference}

\newpage
\appendix

\newpage

\section{Speculative Decoding Algorithm}
\label{app:spec-decoding}

We give the full pseudocode for the standard speculative decoding procedure of \citet{leviathan2023fast, chen2023accelerating}, referenced in \Cref{sec:bg-speculative}.
Given a draft length $K$ and the draft--target pair $(\dist, \p)$, each iteration proceeds in two phases.
In the \textbf{draft} phase, given the string generated so far $\srcStr$, the draft model generates a draft continuation $\draftStr$ of at most $K$ tokens autoregressively from $\prefixq$.
In the \textbf{verification} phase, the target model scores all draft tokens in a single forward pass.
Each draft token $\draftSym_j$ is then tested sequentially using the standard rejection sampling criterion: it is accepted with probability $\accp_j = \min\bigl(1, \prefixp(\draftSym_j \mid \srcStr\,\draftStr_{<j}) / \prefixq(\draftSym_j \mid \srcStr\,\draftStr_{<j})\bigr)$. On the first rejection, the token is resampled from the residual distribution $\max\bigl(0,\; \prefixp(\cdot \mid \srcStr\,\draftStr_{<j}) - \prefixq(\cdot \mid \srcStr\,\draftStr_{<j})\bigr)$ and the draft is truncated.
If every token passes, a bonus token is drawn from $\prefixp(\cdot \mid \srcStr\,\draftStr)$.

\begin{algorithm}[H]
\caption{Speculative Decoding (one round)}\label{alg:spec-decoding}
\begin{algorithmic}[1]
\Require Prefix $\srcStr$, draft model $\dist$, target model $\p$, draft length $K$
\Ensure Extended string $\srcStr\, \draftStr$
\Statex \textbf{--- Draft ---}
\State $\draftStr \gets \varepsilon$
\For{$t = 1, \ldots, K$}
    \State $\draftSym_t \sim \prefixq(\cdot \mid \srcStr\, \draftStr)$
    \State $\draftStr \gets \draftStr\, \draftSym_t$
\EndFor
\Statex \textbf{--- Verification ---}
\For{$j = 1, \ldots, K$}
    \State $u \sim \mathrm{Uniform}(0,1)$
    \If{$u > \min\bigl(1,\; \prefixp(\draftSym_j \mid \srcStr\,\draftStr_{<j})\,/\,\prefixq(\draftSym_j \mid \srcStr\, \draftStr_{<j})\bigr)$} \Comment{reject}
        \State $\draftSym_j \sim \max\bigl(0,\; \prefixp(\cdot \mid \srcStr\,\draftStr_{<j}) - \prefixq(\cdot \mid \srcStr\, \draftStr_{<j})\bigr) / Z_j$ \Comment{$Z_j$ normalizes}
        \State $\draftStr \gets \draftStr_{<j}\, \draftSym_j$ \Comment{truncate and replace}
        \State \textbf{break}
    \ElsIf{$j = K$} \Comment{all tokens accepted}
        \State $\draftSym \sim \prefixp(\cdot \mid \srcStr\, \draftStr)$ \Comment{bonus token}
        \State $\draftStr \gets \draftStr\, \draftSym$
    \EndIf
\EndFor
\State \Return $\srcStr\, \draftStr$
\end{algorithmic}
\end{algorithm}

\newpage

\section{Proofs}
\label{app:proofs}

Throughout this section, $W \defeq \p(\RV)/\dist(\RV)$ denotes the importance weight and $\chi^2(\p \| \dist) \defeq \E_{\dist}[(W - 1)^2]$ the $\chi^2$-divergence.

\subsection{Proof of \Cref{prop:sir-tv}}

We first state the two supporting lemmata, then prove each.

\begin{lemma}\label{lem:unnorm-is}
Let $\srcStr\psup{1}, \ldots, \srcStr\psup{N} \iid \dist$ with $\p \ll \dist$, and let  $\ptilde(A) \defeq \frac{1}{N}\sum_{n=1}^N \frac{\p(\srcStr\psup{n})}{\dist(\srcStr\psup{n})} \indicator{\srcStr\psup{n} \in A}$. Then for each event $A \subseteq \srcStrings$:
\begin{equation}
    \E[(\ptilde(A) - \p(A))^2] \;=\; \frac{1}{N}\left(\sum_{\srcStr \in A}\frac{\p(\srcStr)^2}{\dist(\srcStr)} - \p(A)^2\right) \;\leq\; \frac{1 + \chi^2(\p \| \dist)}{N}.
\end{equation}
\end{lemma}
\begin{proof}
The estimator $\ptilde(A)$ is unbiased ($\E[\ptilde(A)] = \p(A)$), so its MSE equals its variance:
\begin{align}
    \E[(\ptilde(A) - \p(A))^2]
    &= \Var_{\dist}\!\left[\frac{\p(\RV)}{\dist(\RV)} \indicator{\RV \in A}\right] / N
    && \textcolor{gray}{\text{(unbiased; MSE $=$ Var)}} \\
    &= \frac{1}{N}\left(\E_{\dist}\!\left[\frac{\p(\RV)^2}{\dist(\RV)^2} \indicator{\RV \in A}\right] - \p(A)^2\right)
    && \textcolor{gray}{\text{($\Var[\cdot] = \E[\cdot^2] - \E[\cdot]^2$)}} \\
    &= \frac{1}{N}\left(\sum_{\srcStr \in A} \frac{\p(\srcStr)^2}{\dist(\srcStr)} - \p(A)^2\right)
    && \textcolor{gray}{\text{(expanding)}} \\
    &\leq \frac{1}{N}\sum_{\srcStr \in A} \frac{\p(\srcStr)^2}{\dist(\srcStr)}
    && \textcolor{gray}{\text{($\p(A)^2 \geq 0$)}} \\
    &\leq \frac{1}{N}\sum_{\srcStr \in \srcStrings} \frac{\p(\srcStr)^2}{\dist(\srcStr)}
    = \frac{1 + \chi^2(\p \| \dist)}{N}
    && \textcolor{gray}{\text{($A \subseteq \srcStrings$)}} \qedhere
\end{align}
\end{proof}

\begin{lemma}\label{lem:snis-per-string}
Let $\srcStr\psup{1}, \ldots, \srcStr\psup{N} \iid \dist$ with $\p \ll \dist$ and $\E_{\dist}[W^4] < \infty$. Let $\phat(A) \defeq \frac{\ptilde(A)}{\ptilde(\srcStrings)}$ for any $A \subseteq \srcStrings$.
\begin{enumerate}[label=(\alph*)]
    \item\label{item:snis-event} \textbf{(Event-level bound.)} For any event $A \subseteq \srcStrings$,
    \begin{equation}\label{eq:snis-event-bound}
        \E[(\phat(A) - \p(A))^2] \leq \frac{1 + \chi^2(\p \| \dist)}{N} + O(N^{-3/2}).
    \end{equation}
    \item\label{item:snis-singleton} \textbf{(Singleton refinement.)} For each string $\srcStr \in \srcStrings$,
    \begin{align}\label{eq:snis-singleton}
        \E[(\phat(\srcStr) - \p(\srcStr))^2]
        &= \frac{1}{N}\!\Bigl(\frac{\p(\srcStr)^2}{\dist(\srcStr)}
        - 2\p(\srcStr)\frac{\p(\srcStr)^2}{\dist(\srcStr)}
        + \p(\srcStr)^2 (1+\chi^2(\p \| \dist))\Bigr) \notag\\
        &\quad + \frac{C(\srcStr)}{N^{3/2}},
    \end{align}
    where $|C(\srcStr)| \leq c \frac{\p(\srcStr)^3}{\dist(\srcStr)^2}$ for a universal constant $c > 0$.
\end{enumerate}
\end{lemma}
\begin{proof}
We first prove the general event-level bound~(a), then derive the singleton refinement~(b).

\noindent \textbf{Step 1: Setup.}
Let $S \defeq \ptilde(\srcStrings)$ and note $\E[S] = 1$. For any $A \subseteq \srcStrings$, define the importance-weighted estimator $\ptilde(A) = \frac{1}{N}\sum_{n=1}^{N} W\psup{n} \indicator{\srcStr\psup{n} \in A}$ where $W\psup{n} \defeq \p(\srcStr\psup{n})/\dist(\srcStr\psup{n})$, and set
\[
U_A \defeq \ptilde(A) - \p(A), \qquad V \defeq S - 1.
\]
Both $U_A$ and $V$ have mean zero, and $\phat(A) - \p(A) = (U_A - \p(A) V) / (1 + V)$.

\noindent \textbf{Step 2: Bound the Numerator.}
The numerator simplifies to a single sum of i.i.d.\ centered terms:
\begin{equation}\label{eq:numerator-sum}
    U_A - \p(A) V = \frac{1}{N}\sum_{n=1}^{N} W\psup{n}\bigl(\indicator{A}(\srcStr\psup{n}) - \p(A)\bigr).
\end{equation}
Each summand has mean zero under $\dist$ because $\E_{\dist}[W(\indicator{A} - \p(A))] = \p(A) - \p(A) = 0$. Therefore
\begin{equation}\label{eq:numerator-var}
    \E\bigl[(U_A - \p(A) V)^2\bigr] = \frac{1}{N}\,\E_{\dist}\left[\frac{\p(\RV)^2}{\dist(\RV)^2}\bigl(\indicator{A}(\RV) - \p(A)\bigr)^2\right].
\end{equation}
Since $(\indicator{A}(\RV) - \p(A))^2 \leq 1$, this yields
\begin{equation}\label{eq:numerator-upper}
    \E\bigl[(U_A - \p(A) V)^2\bigr] \leq \frac{1}{N}\,\E_{\dist}\left[\frac{\p(\RV)^2}{\dist(\RV)^2}\right] = \frac{1 + \chi^2(\p \| \dist)}{N},
\end{equation}
using $\chi^2(\p \| \dist) = \E_{\dist}[(W - 1)^2] = \E_{\dist}[W^2] - 1$.

\noindent \textbf{Step 3: The Denominator Contributes a Higher-order Correction.}
Write $Z_A \defeq U_A - \p(A) V = \frac{1}{N}\sum_{n=1}^N W\psup{n}(\indicator{A}(\srcStr\psup{n}) - \p(A))$ as in~\eqref{eq:numerator-sum},
and let $Y \defeq W(\indicator{A}(\RV) - \p(A))$ so that $Z_A = \frac{1}{N}\sum_{n=1}^N Y\psup{n}$ with $\E[Y] = 0$.
Because $V$ is unbounded, the power-series expansion $1/(1+V)^2 = \sum_{k=0}^{\infty}(-1)^k(k+1)V^k$ need not converge pointwise.
We handle this by splitting into a \emph{good} event $\mathcal{G} \defeq \{|V| < \tfrac{1}{2}\}$, on which the series converges absolutely, and its complement $\mathcal{B} \defeq \{|V| \geq \tfrac{1}{2}\}$:
\begin{equation}\label{eq:good-bad-split}
    \E\!\left[\frac{Z_A^2}{(1+V)^2}\right]
    = \E\!\left[\frac{Z_A^2}{(1+V)^2}\,\indicator{\mathcal{G}}\right]
    + \E\!\left[\frac{Z_A^2}{(1+V)^2}\,\indicator{\mathcal{B}}\right].
\end{equation}

\noindent\emph{Bounding the bad event.}\;
Note that $S = 1 + V = \frac{1}{N}\sum_{n} W\psup{n} > 0$ a.s.\ (since $\dist(\srcStr) > 0$ whenever $\p(\srcStr) > 0$).
Since $W\psup{n} \geq 0$ and $|\indicator{A}(\srcStr\psup{n}) - \p(A)| \leq 1$, the triangle inequality gives
\begin{equation}\label{eq:za-leq-s}
    |Z_A| \leq \frac{1}{N}\sum_{n=1}^{N} W\psup{n} = S,
\end{equation}
and therefore $Z_A^2/(1+V)^2 = Z_A^2/S^2 \leq 1$ almost surely.
By Markov's inequality applied to $V^4$:
\begin{equation}\label{eq:bad-event-bound}
    \Pr(\mathcal{B})
    = \Pr(|V| \geq \tfrac{1}{2})
    \leq 16\,\E[V^4]
    = 16\!\left(\frac{\E_{\dist}[(W - 1)^4]}{N^3} + \frac{3(N-1)(\chi^2)^2}{N^3}\right)
    = O(N^{-2}),
\end{equation}
where $\E_{\dist}[(W-1)^4] < \infty$ since expanding $(W-1)^4$ yields a polynomial in $W$ of degree four and $\E_{\dist}[W^4] < \infty$.
Combining~\eqref{eq:za-leq-s} and~\eqref{eq:bad-event-bound}:
\begin{equation}\label{eq:bad-contribution}
    \E\!\left[\frac{Z_A^2}{(1+V)^2}\,\indicator{\mathcal{B}}\right]
    \leq \Pr(\mathcal{B})
    = O(N^{-2}).
\end{equation}

\noindent\emph{Series expansion on the good event.}\;
On $\mathcal{G}$, $|V| < \frac{1}{2}$, so the series $1/(1+V)^2 = 1 + \sum_{k=1}^{\infty}(-1)^k(k+1)V^k$ converges absolutely.
We bound the correction by factoring out one power of $|V|$: on $\mathcal{G}$,
\begin{equation}\label{eq:series-tail}
    \left|\sum_{k=1}^{\infty}(-1)^k(k+1)V^k\right|
    \leq |V|\sum_{k=1}^{\infty}(k+1)|V|^{k-1}
    \leq |V|\sum_{k=1}^{\infty}(k+1)\!\left(\tfrac{1}{2}\right)^{k-1}
    \eqqcolon C_0\,|V|,
\end{equation}
where $C_0 = \sum_{k=1}^{\infty}(k+1)2^{-(k-1)} < \infty$ is a numerical constant.
Therefore
\[
    \E\!\left[\frac{Z_A^2}{(1+V)^2}\,\indicator{\mathcal{G}}\right]
    \leq \E[Z_A^2 \indicator{\mathcal{G}}] + C_0\,\E[Z_A^2\,|V|\,\indicator{\mathcal{G}}]
    \leq \E[Z_A^2] + C_0\,\E[Z_A^2\,|V|].
\]
By Cauchy--Schwarz,
\begin{equation}\label{eq:cs-bound}
    \E[Z_A^2\,|V|] \leq \sqrt{\E[Z_A^4]\,\E[V^{2}]}.
\end{equation}
For the $Z_A$ factor, since $(\indicator{A} - \p(A))^2 \leq 1$ and $(\indicator{A} - \p(A))^4 \leq 1$, we have $\E[Y^2] \leq \E_{\dist}[W^2] = 1 + \chi^2$ and $\E[Y^4] \leq \E_{\dist}[W^4]$.
The fourth moment of the sample mean is
\begin{equation}\label{eq:za4}
    \E[Z_A^4] = \frac{\E[Y^4]}{N^3} + \frac{3(N-1)(\E[Y^2])^2}{N^3}
    \leq \frac{\E_{\dist}[W^4]}{N^3} + \frac{3(1 + \chi^2)^2}{N^2} = O(N^{-2}).
\end{equation}
For the $V$ factor, $\E[V^{2}] = \chi^2/N$ since $\Var_{\dist}(W) = \chi^2(\p \| \dist)$.
Substituting into~\eqref{eq:cs-bound}:
\begin{equation}
    \E[Z_A^2\,|V|] \leq \sqrt{O(N^{-2}) \cdot O(N^{-1})} = O(N^{-3/2}).
\end{equation}
Combining the good-event bound with the bad-event bound~\eqref{eq:bad-contribution} and~\eqref{eq:numerator-upper}:
\begin{equation}
    \E\!\left[\frac{Z_A^2}{(1+V)^2}\right]
    \leq \frac{1 + \chi^2(\p \| \dist)}{N} + O(N^{-3/2}).
\end{equation}
Crucially, all constants depend only on $\E_{\dist}[W^4]$ and $\chi^2$, not on $A$.
Combining with~\eqref{eq:numerator-upper} proves part~(a).

\noindent \textbf{Step 4: Singleton Refinement --- Exact Leading Term.}
For the singleton event $A = \{\srcStr\}$, we write $U \defeq U_{\{\srcStr\}} = \ptilde(\srcStr) - \p(\srcStr)$. Expanding $\E[(U - \p(\srcStr)V)^2]  = \E[U^2] - 2\p(\srcStr)\E[UV] + \p(\srcStr)^2\E[V^2]$, we compute
\begin{subequations}
\begin{align}
    \E[U^2] &= \frac{1}{N}\!\left(\frac{\p(\srcStr)^2}{\dist(\srcStr)} - \p(\srcStr)^2\right)
    && \textcolor{gray}{\text{(\Cref{lem:unnorm-is} with $A = \{\srcStr\}$)}} \\
    \E[UV] &= \frac{1}{N}\!\left(\frac{\p(\srcStr)^2}{\dist(\srcStr)} - \p(\srcStr)\right)
    && \textcolor{gray}{\text{(covariance)}} \\
    \E[V^2] &= \frac{\chi^2(\p \| \dist)}{N}
    && \textcolor{gray}{\text{(definition of $\chi^2$)}}
\end{align}
\end{subequations}
Substituting yields
\begin{equation}
    \E[(U - \p(\srcStr)V)^2]
    = \frac{1}{N}\!\left(\frac{\p(\srcStr)^2}{\dist(\srcStr)} - \p(\srcStr)^2
    - 2\p(\srcStr)\!\left(\frac{\p(\srcStr)^2}{\dist(\srcStr)} - \p(\srcStr)\right)
    + \p(\srcStr)^2 \chi^2(\p \| \dist)\right).
\end{equation}

\noindent \textbf{Step 5: Singleton Remainder Bound.}
The remainder $\Rem(\srcStr)$ from the Taylor expansion of $1/(1+V)^2$ in the singleton case satisfies $|\Rem(\srcStr)| \leq c \frac{\p(\srcStr)^3}{\dist(\srcStr)^2 N^{3/2}}$. This follows because the dominant correction $-2\E[(U - \p(\srcStr)V)^2 V]$ involves expanding $(U - \p(\srcStr)V)^2 V = U^2 V - 2\p(\srcStr)UV^2 + \p(\srcStr)^2 V^3$, and the leading term $\E[U^2 V]$ is bounded by
\begin{subequations}
    \begin{align}
    \E[U^2 V] &=   \frac{1}{N^2}  \E_{\dist}\left[\left(\frac{\p(\RV)}{\dist(\RV)}\indicator{\RV = \srcStr} - \p(\srcStr)\right)^{2}\left(\frac{\p(\RV)}{\dist(\RV)} - 1\right)\right] \\
    &\leq \frac{1}{N^2}\E_{\dist}\!\left[\frac{\p(\RV)^3}{\dist(\RV)^3} \indicator{\RV = \srcStr}\right]
    = \frac{\p(\srcStr)^3}{\dist(\srcStr)^2 N^2}.
    \end{align}
\end{subequations}
The remaining terms $\E[UV^2]$ and $\E[V^3]$ are bounded similarly. Setting $C(\srcStr) = N^{3/2} \Rem(\srcStr)$ gives $|C(\srcStr)| \leq c \frac{\p(\srcStr)^3}{\dist(\srcStr)^2}$, completing part~(b). \qedhere
\end{proof}

We now restate and prove the proposition.
\sirtvbound*

\begin{proof}
\noindent\textbf{Item 1} ($L^1$ Bias):
For any $A \subseteq \srcStrings$:
\begin{subequations}
\begin{align}
    (\pbar(A) - \p(A))^2
    &= (\E[\phat(A)] - \p(A))^2
    && \textcolor{gray}{\text{(definition of } \pbar\text{)}} \\
    &\leq \E[(\phat(A) - \p(A))^2]
    && \textcolor{gray}{\text{(Jensen's inequality)}} \\
    &=O\left (\frac{1 + \chi^2(\p \| \dist)}{N} \right)
    && \textcolor{gray}{\text{(\Cref{lem:snis-per-string}\ref{item:snis-event})}} \\
    \implies |\pbar(A) - \p(A)|
    &=O\left (\sqrt{\frac{1 + \chi^2(\p \| \dist)}{N}} \right)
    && \textcolor{gray}{\text{(square root)}}
\end{align}
Since this holds for every $A \subseteq \srcStrings$, we have
\begin{align}
    \BiasOne(\phat)
    = \|\pbar - \p\|_1
    &= 2\,\TVD(\pbar, \p)
    && \textcolor{gray}{\text{($L^1 = 2 \cdot$ TV)}} \\
    &= 2 \sup_{A} |\pbar(A) - \p(A)|
    && \textcolor{gray}{\text{(TV $=$ sup)}} \\
    &= O\!\left(\sqrt{\frac{1 + \chi^2(\p \| \dist)}{N}}\right)
    && \textcolor{gray}{\text{(applying the bound above)}}
\end{align}
\end{subequations}

\noindent\textbf{Item 2} ($L^2$ Bias):
\begin{subequations}
\begin{align}
    \|\pbar - \p\|_2^2
    &= \textstyle\sum_{\srcStr \in \srcStrings} (\pbar(\srcStr) - \p(\srcStr))^2
    && \textcolor{gray}{\text{(summing over $\srcStrings$)}} \\
    &\leq\sum_{\srcStr \in \srcStrings} \E[(\phat(\srcStr) - \p(\srcStr))^2]
    && \textcolor{gray}{\text{(Jensen's inequality)}} \\
    &= \sum_{\srcStr \in \srcStrings} \Biggl(\frac{1}{N}\!\Bigl(\frac{\p(\srcStr)^2}{\dist(\srcStr)}
        - 2\p(\srcStr)\frac{\p(\srcStr)^2}{\dist(\srcStr)}
        + \p(\srcStr)^2 (1+\chi^2(\p \| \dist))\Bigr)
    \notag\\
    &\qquad\qquad\qquad + \frac{C(\srcStr)}{N^{3/2}}\Biggr)
    && \textcolor{gray}{\text{(\Cref{lem:snis-per-string}\ref{item:snis-singleton})}} \\
    &= O\!\left(\frac{1 + \chi^2(\p\|\dist)}{N}\right)
    && \textcolor{gray}{\text{(simplify; $\textstyle\sum_\srcStr \p(\srcStr)^2/\dist(\srcStr) = 1 + \chi^2$)}}
\end{align}
\end{subequations}

\noindent\textbf{Item 3} (MSE):
\begin{subequations}
\begin{align}
    \E[\|\phat - \p\|_2^2]
    &= \sum_{\srcStr \in \srcStrings} \E[(\phat(\srcStr) - \p(\srcStr))^2]
    && \textcolor{gray}{\text{(linearity)}} \\
    &= \sum_{\srcStr \in \srcStrings} \Biggl(\frac{1}{N}\!\Bigl(\frac{\p(\srcStr)^2}{\dist(\srcStr)}
        - 2\p(\srcStr)\frac{\p(\srcStr)^2}{\dist(\srcStr)}
        + \p(\srcStr)^2 (1+\chi^2(\p \| \dist))\Bigr)
    \notag\\
    &\qquad\qquad\qquad + \frac{C(\srcStr)}{N^{3/2}}\Biggr)
    && \textcolor{gray}{\text{(\Cref{lem:snis-per-string}\ref{item:snis-singleton})}} \\
    &= O\left(\frac{1 + \chi^2(\p \| \dist)}{N}\right)
    && \textcolor{gray}{\text{(simplify; $\textstyle\sum_\srcStr \p(\srcStr)^2/\dist(\srcStr) = 1 + \chi^2$)}}
\end{align}
\end{subequations}
In Items~2 and~3 we applied \Cref{lem:snis-per-string}\ref{item:snis-singleton}, whose hypotheses require $\E_{\dist}[W^4] < \infty$.
This is assumed directly; the finiteness of the lower moments $\E_{\dist}[W^2]$ and $\E_{\dist}[W^3]$ (needed for $\chi^2 < \infty$ and the remainder bound) follows by Lyapunov's inequality: $\E_{\dist}[W^r] \leq (\E_{\dist}[W^4])^{r/4} < \infty$ for $r \leq 4$.
This completes all three items.
\end{proof}

\subsection{ESS Rate Results}
\label{app:ess-proofs}

Because the importance weight of a bonus particle depends only on the $K$ draft tokens (the bonus token cancels between the target and proposal densities), the ESS is governed by the $\chi^2$-divergence between the $K$-token \emph{block} distributions $\p_K$ and $\dist_K$---not the $(K\!+\!1)$-token distributions.

\begin{theorem}[ESS Rate]\label{thm:ess-rate}
Under the assumption $\chi^2(\p_K \| \dist_K) < \infty$, the ESS fraction converges almost surely:
\begin{equation}\label{eq:ess-limit}
  \frac{\ESS_N}{N} \;\xrightarrow{\;\mathrm{a.s.}\;}\; \frac{1}{1 + \chi^2(\p_K \| \dist_K)}
\end{equation}
Crucially, the divergences are between the $K$-token block distributions, even though each particle produces $K + 1$ tokens.
\end{theorem}

\begin{proof}[Proof of \Cref{thm:ess-rate}]
By the strong law of large numbers,
$\frac{1}{N}\sum_{i=1}^N \Weight_i \xrightarrow{\mathrm{a.s.}} \E_{\dist_K}[\Weight] = 1$
and
$\frac{1}{N}\sum_{i=1}^N \Weight_i^2 \xrightarrow{\mathrm{a.s.}} \E_{\dist_K}[\Weight^2] = 1 + \chi^2(\p_K \| \dist_K)$.
Therefore
\begin{equation}
  \frac{\ESS_N}{N} = \frac{\bigl(\tfrac{1}{N}\sum_i \Weight_i\bigr)^2}{\tfrac{1}{N}\sum_i \Weight_i^2} \;\xrightarrow{\;\mathrm{a.s.}\;}\; \frac{(\E_{\dist_K}[\Weight])^2}{\E_{\dist_K}[\Weight^2]} = \frac{1}{1 + \chi^2(\p_K \| \dist_K)}.
\end{equation}
\end{proof}

If (as a simplified model) the $K$ tokens are independent and identically distributed, the $\chi^2$-divergence factors as an exact power:

\begin{proposition}[I.I.D.\ Case]\label{prop:iid-ess}
Suppose $\p_K = \p^{\otimes K}$ and $\dist_K = \dist^{\otimes K}$ (i.i.d.\ tokens)\footnote{The exponential decay in \Cref{prop:iid-ess} assumes i.i.d.\ tokens and is likely loose for autoregressive models, where successive tokens are strongly correlated given the prefix.
The exact ESS limit in \Cref{thm:ess-rate} holds without this assumption; understanding how the block $\chi^2$-divergence $\chi^2(\p_K \| \dist_K)$ grows with $K$ under realistic autoregressive dependencies is an interesting direction for future work.}. Then
\begin{equation}\label{eq:iid-chi2}
  1 + \chi^2(\p_K \| \dist_K) = \bigl(1 + \chi^2(\p \| \dist)\bigr)^{K}.
\end{equation}
Therefore $\ESS_N / N \to \bigl(1 + \chi^2(\p \| \dist)\bigr)^{-K}$.
Each particle produces $K + 1$ tokens but the ESS penalty is only $\bigl(1 + \chi^2(\p \| \dist)\bigr)^{-K}$, not $\bigl(1 + \chi^2(\p \| \dist)\bigr)^{-(K+1)}$.
\end{proposition}

\begin{proof}[Proof of \Cref{prop:iid-ess}]
By independence,
\begin{equation}
\E_{\dist_K}[\Weight^2]
= \sum_{\srcStr_{1:K}} \frac{\p_K(\srcStr_{1:K})^2}{\dist_K(\srcStr_{1:K})}
= \prod_{t=1}^{K} \sum_{\srcSym} \frac{\p(\srcSym)^2}{\dist(\srcSym)}
= \bigl(1 + \chi^2(\p \| \dist)\bigr)^{K}. \qedhere
\end{equation}
\end{proof}

\begin{corollary}[Bonus Token Efficiency Gain]\label{cor:bonus-gain}
In the i.i.d.\ setting, without the bonus token, generating $K + 1$ tokens entirely from the draft model would give ESS fraction $\bigl(1 + \chi^2(\p \| \dist)\bigr)^{-(K+1)}$. The bonus token improves this by a multiplicative factor of $\bigl(1 + \chi^2(\p \| \dist)\bigr)$:
\begin{equation}\label{eq:bonus-gain}
  \frac{(\ESS / N)_{\mathrm{with\;bonus}}}{(\ESS / N)_{\mathrm{without\;bonus}}} = 1 + \chi^2(\p \| \dist).
\end{equation}
\end{corollary}

\begin{proof}[Proof of \Cref{cor:bonus-gain}]
Without the bonus token, all $K+1$ tokens are drafted from $\dist$, so the ESS fraction converges to $\bigl(1 + \chi^2(\p \| \dist)\bigr)^{-(K+1)}$ by \Cref{prop:iid-ess} applied with draft length $K+1$.
With the bonus token, the ESS fraction converges to $\bigl(1 + \chi^2(\p \| \dist)\bigr)^{-K}$ by \Cref{prop:iid-ess}.
The ratio is
\begin{equation}
  \frac{\bigl(1 + \chi^2(\p \| \dist)\bigr)^{-K}}{\bigl(1 + \chi^2(\p \| \dist)\bigr)^{-(K+1)}} = 1 + \chi^2(\p \| \dist). \qedhere
\end{equation}
\end{proof}

\section{Speed-up Summary Table}
\label{app:speedup-table}

\begin{table}[H]
\caption{SMC-SD speed-up over standard SD at two operating points.
\emph{Iso-accuracy}: fastest SMC-SD config within 3pp of SD accuracy.
\emph{Max speed}: fastest config within 10pp. The Llama configuration uses a 1B draft model and 8B target model for both SMC-SD and standard SD. The Qwen configuration uses a 0.5B draft model and 14B target for standard SD and a 3B draft and 14B target for SMC-SD.}
\label{tab:speedup-summary}
\centering
\small
\begin{tabular}{@{}ll cc cc c cc c@{}}
\toprule
& & \multicolumn{2}{c}{SD} & \multicolumn{3}{c}{SMC-SD (iso-accuracy)} & \multicolumn{3}{c}{SMC-SD (max speed)} \\
\cmidrule(lr){3-4} \cmidrule(lr){5-7} \cmidrule(lr){8-10}
Benchmark & Models & Acc & TPS & Acc & TPS & Speed & Acc & TPS & Speed \\
\midrule
GSM8K & Llama & 78 & 250 & 81 & 290 & $1.16\times$ & 72 & 380 & $1.52\times$ \\
 & Qwen & 90 & 170 & 90 & 233 & $1.37\times$ & 88.5 & 241 & $1.42\times$ \\
\midrule
MATH500 & Llama & 52 & 250 & 53 & 380 & $1.52\times$ & 53 & 380 & $1.52\times$ \\
 & Qwen & 75 & 200 & 79 & 240 & $1.20\times$ & 74 & 270 & $1.35\times$ \\
\midrule
AlpacaEval & Llama & 89 & 190 & 88 & 270 & $1.42\times$ & 84 & 350 & $1.84\times$ \\
 & Qwen & 92 & 112 & 90 & 210 & $1.88\times$ & 88 & 250 & $2.23\times$ \\
\midrule
DS1000 & Llama & 26.5 & 150 & 26.5 & 379.3 & $2.53\times$ & 23.2 & 502.9 & $3.35\times$ \\
 & Qwen & 35 & 119 & 33 & 244 & $2.05\times$ & 31 & 300 & $2.52\times$ \\
\bottomrule
\end{tabular}
\end{table}

\section{Hardware Utilization and Speedup Proofs}
\label{app:smc-sd-analysis}

Let $\mathcal{P}$ denote the target model's parameter count, $b$ the bytes per parameter, $\rho = \mathcal{P}_{\dist}/\mathcal{P}_{\p}$ the draft-to-target parameter ratio, $c = T_{\dist}/T_{\p} \ll 1$ the draft-to-target decoding-time ratio.
Let $B$ denote the number of independent user requests served concurrently (the \emph{batch size}), and let $BN$ denote the \emph{effective batch size}---the total number of particles across all requests.
In the memory-bound regime $c \approx \rho$.
In one SMC-SD step, the draft model runs $K$ sequential forward passes (one per token position, with all $BN$ particles batched), taking $cT_{\p}K$ seconds.
The target model scores all $BN(K{+}1)$ tokens (including bonus tokens) in a single forward pass.
The resampling step is negligible.

\subsection{Arithmetic Intensity}

\begin{lemma}[Target-pass arithmetic intensity]\label{lem:ai-target}
The arithmetic intensity of the target model's forward pass in one SMC-SD cycle is
\begin{equation}\label{eq:app-ai-target}
    \text{AI}_{\text{target}} = \frac{2BN(K{+}1)}{b}\;\;\text{FLOPs/byte}.
\end{equation}
\end{lemma}

\begin{proof}
Using the standard approximation that a forward pass through a model with $\mathcal{P}$ parameters requires $2\mathcal{P}$ FLOPs (one multiply-add per parameter), the target pass processes $BN(K{+}1)$ tokens in one invocation---$BN K$ draft tokens plus $BN$ bonus tokens carried over from the previous round---for a total of $2\mathcal{P} \cdot BN(K{+}1)$ FLOPs.
The bytes transferred are $b\mathcal{P}$ (loading the weights once), giving $\text{AI}_{\text{target}} = 2\mathcal{P} \cdot BN(K{+}1) / (b\mathcal{P}) = 2BN(K{+}1)/b$.
For FP16 ($b{=}2$) this reduces to $BN(K{+}1)$.
For comparison, autoregressive decoding achieves $\text{AI} = 2B/b$ and standard SD achieves $\text{AI} = 2B(K{+}1)/b$ when all $K$ draft tokens are accepted.
When $\text{AI}_{\text{target}}$ exceeds the hardware's ridge point (peak FLOPs / peak bandwidth), the target pass becomes compute-bound and further increasing the effective batch size $BN$ no longer comes for free.
\end{proof}

\begin{lemma}[Full-cycle arithmetic intensity]\label{lem:ai-cycle}
Over a full SMC-SD cycle (including $K$ draft forward passes), the arithmetic intensity is
\begin{equation}\label{eq:app-ai-cycle}
    \text{AI}_{\text{cycle}} = \frac{2BN(\rho K{+}K{+}1)}{b(\rho K + 1)} \;\approx\; \frac{2BN(K{+}1)}{b} \quad\text{when } \rho \ll 1.
\end{equation}
\end{lemma}

\begin{proof}
The total FLOPs over one cycle are $2\mathcal{P}_{\dist} \cdot BN K + 2\mathcal{P} \cdot BN(K{+}1) = 2\mathcal{P}\,BN(\rho K + K + 1)$: the draft model performs $K$ forward passes each processing $BN$ tokens, and the target model performs one pass processing $BN(K{+}1)$ tokens.
The total bytes transferred are $b\mathcal{P}_{\dist} \cdot K + b\mathcal{P} = b\mathcal{P}(\rho K + 1)$: the draft weights are loaded $K$ times and the target weights once.
Dividing gives $\text{AI}_{\text{cycle}} = 2BN(\rho K + K + 1) / (b(\rho K + 1))$.
When $\rho \ll 1$, the draft-model terms become negligible and this simplifies to $2BN(K{+}1)/b$.
\end{proof}

\subsection{Speedup Proofs}

\begin{lemma}[Memory-bound speed-up]\label{lem:speedup-mem}
If the target forward pass is memory bandwidth-bound (i.e.\ $BN(K{+}1) \leq R$, where $R$ is the hardware ridge-point batch size), the SMC-SD speed-up over autoregressive decoding is
\begin{equation}\label{eq:app-speedup}
    S^{(\text{mem})} = \frac{K+1}{cK + 1},
\end{equation}
independent of the effective batch size $BN$.
\end{lemma}

\begin{proof}
When the target forward pass is memory bandwidth-bound, loading the model weights dominates and the verification cost is approximately constant regardless of $BN$.
The cycle time is
\begin{equation}\label{eq:app-cycle-time}
    t_{\text{SMC-SD step}}^{(\text{mem})} = T_{\p}\bigl(cK + 1\bigr).
\end{equation}
Since autoregressive decoding takes $T_{\p}K$ seconds for $K$ tokens and each SMC-SD cycle produces $K{+}1$ tokens per request, the speed-up is $S^{(\text{mem})} = (K{+}1)/(cK+1)$.
Adding more particles or requests is ``free'' because the target weights are loaded only once regardless of $BN$.
\end{proof}

\begin{lemma}[Compute-bound speed-up]\label{lem:speedup-comp}
If the target forward pass is compute-bound (i.e.\ $BN(K{+}1) > R$), the SMC-SD speed-up over autoregressive decoding is
\begin{equation}\label{eq:speedup-compute}
    S^{(\text{comp})} = \frac{K+1}{\hat{c}(BN)\,K + BN(K{+}1)/R},
\end{equation}
where $\hat{c}(BN) = \max(c,\, c \cdot BN/R)$ is the effective draft cost ratio and $R$ is the hardware ridge-point batch size in tokens.
\end{lemma}

\begin{proof}
As $BN(K{+}1)$ grows, the arithmetic intensity of the target pass eventually exceeds the hardware ridge point, and the forward pass becomes compute-bound.
In this regime, the target verification time scales linearly with the number of tokens processed.
Let $T_{\p}^{(1)}$ denote the (memory-bound) time to decode a single token with the target model.
The compute-bound verification of $BN(K{+}1)$ tokens takes approximately $T_{\p}^{(1)} \cdot BN(K{+}1) / R$ seconds, where $R = \text{peak FLOPs}\,/\,\text{peak bandwidth} \cdot b/2$ tokens (with $b$ bytes per parameter).
The draft model, which processes $BN$ tokens in parallel at each of its $K$ steps, may itself become compute-bound if $BN > R$.
In the compute-bound regime, the draft model's per-step cost scales as $c \cdot BN/R$ relative to $T_{\p}^{(1)}$, since the draft model has $c \approx \rho$ times the cost of the target per token.
Writing $\hat{c}(BN) = \max(c,\, c \cdot BN/R)$ for the effective draft cost ratio, the cycle time becomes
\begin{equation}\label{eq:cycle-time-compute}
    t_{\text{SMC-SD step}}^{(\text{comp})} = T_{\p}^{(1)}\!\left(\hat{c}(BN)\,K + \frac{BN(K{+}1)}{R}\right).
\end{equation}
Since autoregressive decoding takes $T_{\p}^{(1)}$ per token, the speed-up is $S^{(\text{comp})} = (K{+}1)/(\hat{c}(BN)K + BN(K{+}1)/R)$.
\end{proof}

\section{Speed-up Under Declining Acceptance Rate}
\label{app:temp-sensitivity}

Standard SD's throughput is tightly coupled to the draft--target acceptance rate, which drops as sampling temperature increases (or more generally, as the draft and target distributions diverge).
SMC-SD, by contrast, always outputs $K{+}1$ tokens per round regardless of acceptance, so its throughput is largely temperature-invariant.
\Cref{fig:speedup-vs-temp} plots the SMC-SD\,/\,SD throughput ratio across temperatures: the speed-up grows from ${\sim}1.5\times$ at $T{=}0.2$ to ${\sim}3\times$ at $T{=}1.0$.
On average across temperatures, SMC-SD sequences incur only a ${\sim}5\%$ increase in negative log-likelihood under the target model compared to exact SD, confirming that the approximation quality remains high even as the throughput advantage widens.

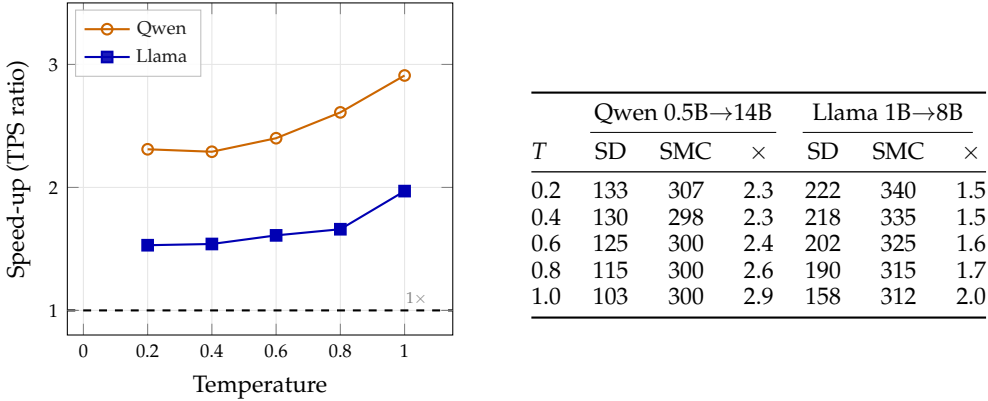
\begin{figure}[H]
\centering
\begin{minipage}[c]{0.48\textwidth}
\centering
\begin{tikzpicture}
\begin{axis}[
    width=\textwidth, height=6cm,
    xlabel={Temperature},
    ylabel={Speed-up (TPS ratio)},
    xmin=-0.05, xmax=1.15,
    ymin=0.8, ymax=3.5,
    xtick={0.0, 0.2, 0.4, 0.6, 0.8, 1.0},
    grid=major,
    grid style={gray!20},
    legend style={font=\scriptsize, at={(0.02,0.98)}, anchor=north west,
                  draw=gray!50, fill=white, fill opacity=0.9},
    every axis label/.style={font=\small},
    tick label style={font=\scriptsize},
    extra y ticks={1},
    extra y tick labels={},
    extra y tick style={grid style={black, dashed, thick}},
]
\addplot[thick, mark=o, mark size=2pt, orange!80!black]
    coordinates {(0.2,2.31)(0.4,2.29)(0.6,2.40)(0.8,2.61)(1.0,2.91)};
\addlegendentry{Qwen}
\addplot[thick, mark=square*, mark size=2pt, blue!70!black]
    coordinates {(0.2,1.53)(0.4,1.54)(0.6,1.61)(0.8,1.66)(1.0,1.97)};
\addlegendentry{Llama}
\node[font=\tiny, anchor=east, gray] at (axis cs:1.1, 1.1) {$1\times$};
\end{axis}
\end{tikzpicture}
\end{minipage}%
\hfill
\begin{minipage}[c]{0.50\textwidth}
\centering
\small
\begin{tabular}{@{}l cc c cc c@{}}
\toprule
& \multicolumn{3}{c}{Qwen 0.5B$\rightarrow$14B} & \multicolumn{3}{c}{Llama 1B$\rightarrow$8B} \\
\cmidrule(lr){2-4} \cmidrule(lr){5-7}
$T$ & SD & SMC & $\!\times\!$ & SD & SMC & $\!\times\!$ \\
\midrule
0.2 & 133 & 307 & 2.3 & 222 & 340 & 1.5 \\
0.4 & 130 & 298 & 2.3 & 218 & 335 & 1.5 \\
0.6 & 125 & 300 & 2.4 & 202 & 325 & 1.6 \\
0.8 & 115 & 300 & 2.6 & 190 & 315 & 1.7 \\
1.0 & 103 & 300 & 2.9 & 158 & 312 & 2.0 \\
\bottomrule
\end{tabular}
\label{tab:temp-tps}
\end{minipage}
\caption{SMC-SD\,/\,SD throughput ratio vs.\ temperature (left) and raw throughput in tokens/s (right).
SMC-SD maintains near-constant TPS while SD degrades as acceptance rates fall, widening the speed-up from ${\sim}1.5\times$ at $T{=}0.2$ to ${\sim}3\times$ at $T{=}1.0$.
On average, SMC-SD sequences incur only a ${\sim}5\%$ increase in NLL under the target model.}
\label{fig:speedup-vs-temp}
\end{figure}

\section{Power Sampling}
\label{app:power-sampling}
One advantage of SMC is that it allows us to sample approximately from un-normalized target distributions whose normalizing constants are intractable. One such target distribution is the \emph{power distribution} \citep{karan2025reasoningsamplingbasemodel}, given by
$\p_{\textrm{pow}}(\srcStr)\propto \p(\srcStr)^\alpha$
where $\srcStr \in \srcStrings$. In contrast to temperature sampling, which tempers the distribution of each \emph{token}, power sampling tempers the distribution over entire \emph{strings}. Applying SMC-SD to sample approximately from this distribution is simply a matter of weighting the draft model $\dist$'s generated tokens by
\begin{equation}
    \Weight_t\psup{n}=\frac{\prefixp(\srcSym_t\psup{n} \mid \srcStr\psup{n}_{<t})^\alpha}{\prefixq(\srcSym_t\psup{n} \mid \srcStr\psup{n}_{<t})}.
\end{equation}
Prior work has proposed using SMC to sample approximately from the power distribution. PowerSMC \citep{azizi2026powersmclowlatencysequencelevelpower} proposes generating particles via temperature sampling
\begin{equation}
    \srcSym_t\psup{n}\sim \frac{\prefixp(\cdot \mid \srcStr\psup{n}_{<t})^\alpha}{\sum_{\srcSym \in \srcAlphabet} \prefixp(\srcSym \mid \srcStr\psup{n}_{<t})^\alpha},
\end{equation}
weighting with importance weight equal to
\begin{equation}
    \Weight_t\psup{n}=\sum_{\srcSym \in \srcAlphabet} \prefixp(\srcSym \mid \srcStr\psup{n}_{<t})^{\alpha},
\end{equation}
and resampling after each batch of tokens is generated (if the ESS is less than a cutoff). The PowerSMC procedure can be expressed as a special case of the SMC-SD framework where $K=1$, the bonus token is removed, and the target and draft models are the same. The authors prove that temperature sampling from the target model provides the locally variance-minimizing proposal for sampling from the power distribution.  \Cref{fig:power_sampling} shows that sampling from the power distribution of an 8B Llama model with a 1B Llama draft model yields a 10\% higher Pass@1 accuracy over the baseline target model. Accuracy for Pass@$k$ for $k>1$ is greater than or equal to the baseline target model as well.

\begin{figure}[H]
    \centering
    \includegraphics[width=1\linewidth]{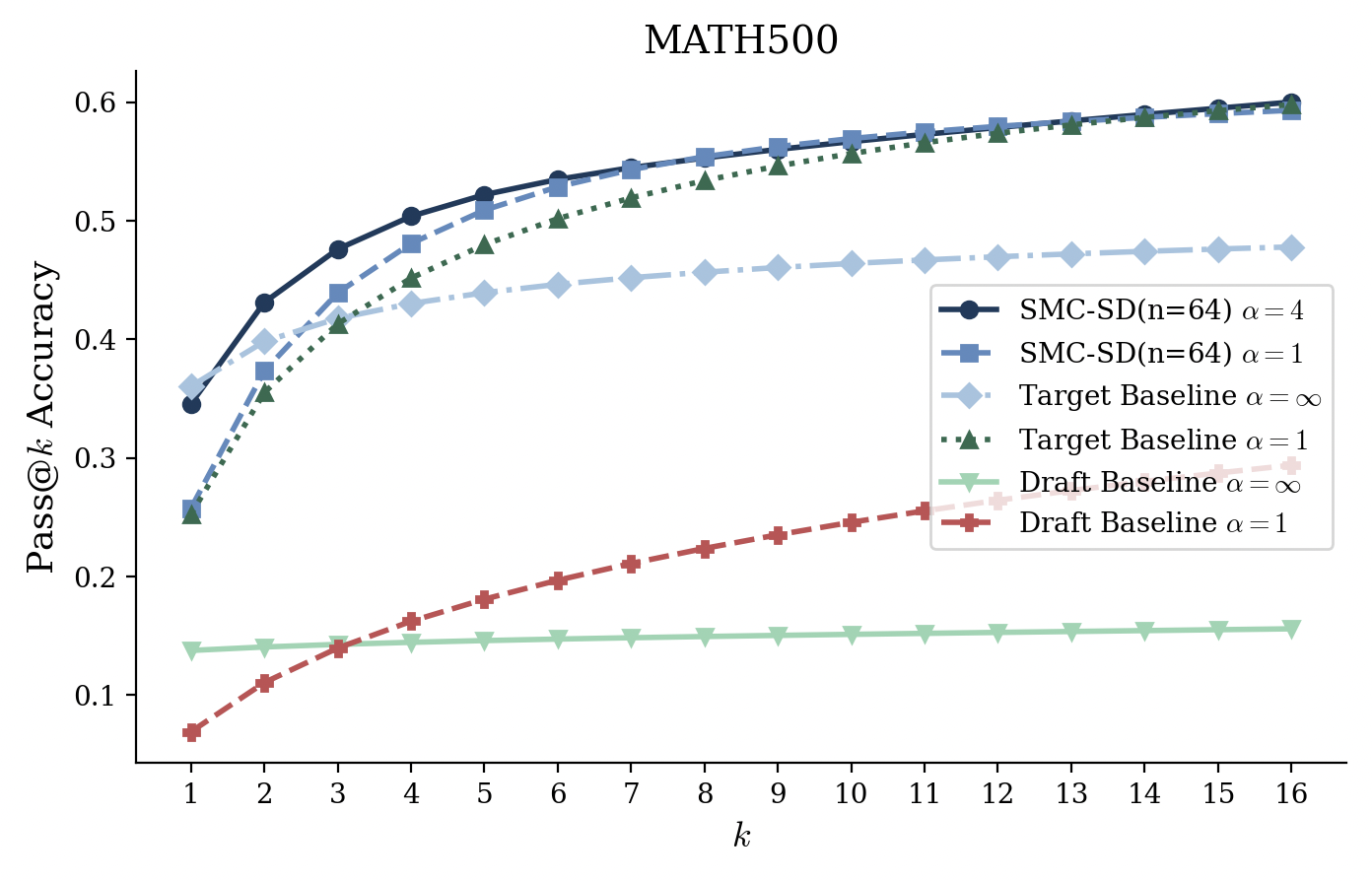}
    \caption{Power Sampling helps improve Pass@$k$ quality}
    \label{fig:power_sampling}
\end{figure}

\section{More Qualitative Examples}
\label{app:qualitative}

\paragraph{GSM8K.}
\begin{description}[leftmargin=2.5em, style=nextline, font=\normalfont]
\item[Llama 1B$\rightarrow$8B]
\textbf{Question:} Claire makes a 3 egg omelet every morning for breakfast. How many dozens of eggs will she eat in 4 weeks?

\textbf{SMC-SD:} Step 1: Convert 4 weeks to days --- 4 weeks $\times$ 7 days/week $=$ 28 days. Step 2: Calculate the number of eggs in 28 days --- 28 days $\times$ 3 eggs/day $=$ 84 eggs. Step 3: Convert the number of eggs from eggs to dozens --- 84 eggs / 12 eggs/dozen $=$ 7 dozens. The final answer is: 7.

\item[Qwen 3B$\rightarrow$14B]
\textbf{Question:} A new program had 60 downloads in the first month. The number of downloads in the second month was three times as many as the downloads in the first month, but then reduced by 30\% in the third month. How many downloads did the program have total over the three months?

\textbf{SMC-SD:} First month: 60 downloads. Second month: $3 \times 60 = 180$ downloads. Third month: reduced by 30\%, so $180 \times 0.70 = 126$ downloads. Total: $60 + 180 + 126 = 366$.
\end{description}

\paragraph{MATH500.}
\begin{description}[leftmargin=2.5em, style=nextline, font=\normalfont]
\item[Llama 1B$\rightarrow$8B --- Counting \& Probability, Level 5]
\textbf{Problem:} Pat is to select six cookies from a tray containing only chocolate chip, oatmeal, and peanut butter cookies. There are at least six of each of these three kinds of cookies on the tray. How many different assortments of six cookies can be selected?

\textbf{SMC-SD:} We can use combinations with repetition. Distributing 6 identical items into 3 distinct groups: $C(n + k - 1, k) = C(3 + 6 - 1, 6) = C(8, 6) = 8!/(6! \times 2!) = 28$.

\item[Qwen 3B$\rightarrow$14B --- Algebra, Level 4]
\textbf{Problem:} Evaluate $\lceil(3.6)^2\rceil - (\lceil 3.6\rceil)^2$.

\textbf{SMC-SD:} $(3.6)^2 = 12.96$, $\lceil 12.96\rceil = 13$, $\lceil 3.6\rceil = 4$, $4^2 = 16$, $13 - 16 = -3$.
\end{description}

\paragraph{AlpacaEval.}
\begin{description}[leftmargin=2.5em, style=nextline, font=\normalfont]
\item[Llama 1B$\rightarrow$8B --- Geography]
\textbf{Instruction:} Where is Indonesia?

\textbf{SMC-SD:} Indonesia is an archipelago of over 17,000 islands located in Southeast Asia. It is situated in the Pacific Ocean and is the world's largest archipelago, both in terms of the number of islands and land area. Indonesia is bordered by the Pacific Ocean to the east, the Indian Ocean to the south, and the South China Sea to the west. It also borders Malaysia and Brunei to the north, and the Philippines to the east. Indonesia is the world''s fourth most populous country, with over 273 million people, and its capital and largest city is Jakarta.

\item[Llama 1B$\rightarrow$8B --- Travel math]
\textbf{Instruction:} If it takes 8 hours to fly to Singapore, stop over for an hour and a half and then fly 14 hours to London, what is the total travel time if the plane is delayed in Singapore?

\textbf{SMC-SD:} To calculate the total travel time, let's break down the journey: Flight from the USA to Singapore: 8 hours. Stopover in Singapore: 1.5 hours. Flight from Singapore to London: 14 hours. Total travel time: $8 + 1.5 + 14 = 23.5$ hours. However, you mentioned the plane is delayed in Singapore. Let's assume the delay is 1 hour\ldots{} Total with delay: 24.5 hours.

\item[Llama 1B$\rightarrow$8B --- Announcement tweet]
\textbf{Instruction:} Write an announcement tweet for medium.com readers about the new blogpost on ``Open Assistant is open source ChatGPT that you don't wanna miss out''.

\textbf{SMC-SD:} ``EXCLUSIVE Launch Day Alert! Don't miss out on the surprise of the year! Our brand new blog post is LIVE on Medium: `Open Assistant is open source ChatGPT that you don't wanna miss out'. Read now and discover the future of conversational AI! [link] \#OpenAssistant \#ChatGPT \#AI''

\item[Qwen 3B$\rightarrow$14B --- Argentina]
\textbf{Instruction:} What language do Argentina people speak?

\textbf{SMC-SD:} The Argentinian people speak Spanish, which is the official language of Argentina. However, in some regions, there may be smaller dialects or local languages that are spoken, but they are not as widely used or respected as Spanish.

\item[Qwen 3B$\rightarrow$14B --- Assess credibility]
\textbf{Instruction:} How can you assess the credibility of a source of information, such as a news article or blog post, without relying solely on the reputation of the author or publisher?

\textbf{SMC-SD:} Assessing the credibility of a source of information, including a news article or blog post, is a matter of taking a few key steps to understand the reliability and quality of the information presented. Here''s a checklist of steps that can be followed to evaluate the credibility of the source:
\begin{itemize}[nosep]
\item \textbf{Check the Author and Publisher's Background:} Look for the author's credentials, expertise, and the publication's reputation. A reputable author or publisher is less likely to have a history of making false claims or spreading misinformation.
\item \textbf{Review the Bias and Methodology:} Assess the sources of information used in the article or blog post. Check if the data or other sources are clearly stated and verified. Look for indications of bias---bias can come from multiple sources, leading to an incomplete or potentially misleading narrative.
\item \textbf{Check for Reproducibility of the Data:} Determine whether the data used in the article or blog post can be reproduced. This is crucial to ensure that the information presented is verifiable. If the data could be easily modified, then the information is less credible.
\item \textbf{Evaluate the Relevance to the Topic:} The article should be based on issues of relevance and current importance. The information should address a specific problem or discuss an issue of public interest. The author must also support their claims with facts and evidence.
\item \textbf{Check for Peer Reactions and Expert Consensus:} The credibility of the source also depends on the extent to which other experts in the field agree with the information. Support from respected figures can significantly bolster the credibility of a source.
\item \textbf{Use External References and Other Sources:} Look for credible references that support the claims made in the article or blog post. Cross-referencing can help verify the accuracy of the information presented.
\item \textbf{Look for a Balanced View:} Ensure that the article or blog post provides a fair and balanced view of the topic. Avoid articles that are overly negative or sensationalistic.
\item \textbf{Check for Up to Date Information:} Keep up with the latest developments in the field or topic being discussed. This can often provide the most reliable and up-to-date information.
\end{itemize}
By following these steps, you can effectively assess the credibility of a source of information, ensuring that you obtain accurate and trustworthy information.

\item[Qwen 3B$\rightarrow$14B --- Music theory subtopics]
\textbf{Instruction:} Create a list of subtopics for the given topic: Music theory.

\textbf{SMC-SD:}
\begin{itemize}[nosep]
\item Fundamental Concepts in Music Theory
\item Rhythm and Timing in Music Theory
\item Tonality and Key Systems in Music Theory
\item Chord and Harmony in Music Theory
\item Ear Training and Musical Sensation in Music Theory
\item Music Instrumentation and Improvisation in Music Theory
\item Music Analysis and Composition in Music Theory
\item Genre and Style in Music Theory
\item Historical Analysis of Music Theory
\item Theory of Scale and Melody in Music Theory
\end{itemize}
\end{description}

\paragraph{DS1000.}
\begin{description}[leftmargin=2.5em, style=nextline, font=\normalfont]
\item[Llama 1B$\rightarrow$8B --- Matplotlib, Problem 511]
\textbf{Prompt:} Given arrays \texttt{x} and \texttt{y}, plot \texttt{x} vs.\ \texttt{y} and label the curve using \texttt{"x-y"} in the legend.

\textbf{SMC-SD:}
\begin{verbatim}
import numpy as np
import matplotlib.pyplot as plt

# Define the arrays
x = np.array([1, 2, 3, 4, 5, 6, 7, 8, 9, 10])
y = np.array([10, 9, 8, 7, 6, 5, 4, 3, 2, 1])

# Create the plot
plt.plot(x, y, label='x-y')
plt.xlabel('x')
plt.ylabel('y')
plt.title('Plot of x vs. y')
plt.legend()

# Show the plot
plt.show()
\end{verbatim}

\item[Qwen 3B$\rightarrow$14B --- NumPy, Problem 292]
\textbf{Prompt:} I want to figure out how to remove \texttt{nan} values from my array. For example, my array looks something like \texttt{x = [1400, 1500, 1600, nan, nan, nan, 1700]}. How can I remove the \texttt{nan} values from \texttt{x} to get something like \texttt{[1400, 1500, 1600, 1700]}?

\textbf{SMC-SD:}
\begin{verbatim}
import numpy as np

# Given array with nan values
x = np.array([1400, 1500, 1600, np.nan, np.nan, np.nan, 1700])

# Filter out nan values
x_clean = x[~np.isnan(x)]

print(x_clean)
\end{verbatim}
\end{description}

\end{document}